\documentclass[10pt]{article}

\usepackage{amsmath,amsfonts,bm}

\def\eqref#1{(\ref{#1})}

\def\floor#1{\lfloor #1 \rfloor}
\def\1{\bm{1}}

\DeclareMathAlphabet{\mathsfit}{\encodingdefault}{\sfdefault}{m}{sl}
\SetMathAlphabet{\mathsfit}{bold}{\encodingdefault}{\sfdefault}{bx}{n}

\newcommand{\E}{\mathbb{E}}

\newcommand{\Cov}{\mathrm{Cov}}

\DeclareMathOperator*{\argmin}{arg\,min}

\newcommand{\infbound}{R^2}
\newcommand{\defeq}{\stackrel{\mathrm{def}}{=}}

\renewcommand{\vec}[1]{\mathbf{#1}}
\newcommand{\err}[2]{\textrm{err}_{#2}^{(#1)}}
\newcommand{\mat}[1]{\mathbf{#1}}

\newcommand{\Rd}{\mathbb{R}^d}

\renewcommand{\a}{\vec{a}}
\newcommand{\A}{\mat{A}}

\newcommand{\bb}{\vec{b}}
\newcommand{\B}{\mat{B}}

\newcommand{\ws}{\vec{w}^*}
\newcommand{\w}{\vec{w}}
\newcommand{\x}{\vec{x}}

\newcommand{\var}[1]{v^{\left(#1\right)}}

\newcommand{\iprod}[2]{\left\langle #1, #2 \right\rangle}

\renewcommand{\E}[1]{\mathbb{E}\left[#1\right]}

\renewcommand{\Cov}{\mat{H}}
\newcommand{\V}{\mat{V}}

\newcommand{\red}[1]{{\color{red} #1}}
\newcommand{\rahul}[1]{}
\newcommand{\rong}[1]{}

\newcommand{\eye}{\Id}
\newcommand{\Id}{\mathbf{I}}
\newcommand{\eqdef}{\stackrel{\textrm{def}}{=}}

\newcommand{\order}[1]{\mathcal{O}\left(#1\right)}

\newcommand{\osgd}{OSGD}

\newcommand{\cifar}{cifar-10}

\newcommand{\lamminH}{\lambda_{\textrm{min}}(\Cov)}

\newcommand{\norm}[1]{\left\| #1 \right\|}

\newcommand{\twonorm}[1]{\norm{#1}_2}

\renewcommand{\c}{c}

\usepackage{times}
\usepackage[margin=1in]{geometry}
\usepackage{natbib,graphics,graphicx,authblk} 
\usepackage{amsfonts,algorithm2e,algpseudocode,amsmath,amssymb,tablefootnote}
\PassOptionsToPackage{hyphens}{url}
\usepackage{amsthm,flushend,fullpage,adjustbox,float,subfigure}
\usepackage{pifont,color}
\usepackage{caption}
\usepackage{multicol}

\usepackage{enumitem,scalerel}
\usepackage[colorlinks = true,
            linkcolor = red,
            urlcolor  = blue,
            citecolor = blue,
            anchorcolor = blue]{hyperref}
\graphicspath{{figures/newer-figs/}}
\title{The Step Decay Schedule: A Near Optimal, Geometrically Decaying Learning Rate Procedure For Least Squares\footnote{This paper appears in the proceedings of the conference in Neural Information Processing Systems (NeurIPS), $2019$, held in Vancouver Canada.}}

\author[1]{Rong Ge}
\author[2]{Sham M. Kakade}
\author[3]{Rahul Kidambi}
\author[4]{Praneeth Netrapalli}
\affil[1]{Duke University, Durham, NC, USA,    \url{rongge@cs.duke.edu}}
\affil[2]{University of Washington, Seattle, WA, USA,  \url{sham@cs.washington.edu}}
\affil[3]{Cornell University, Ithaca, NY, USA,  \url{rkidambi@cornell.edu}}
\affil[4]{Microsoft Research, Bangalore, KA, India, \url{praneeth@microsoft.com}.}
\date{}

\newtheorem{claim}{Claim}

\newtheorem{theorem}{Theorem}
\newtheorem{lemma}[theorem]{Lemma}

\newtheorem{proposition}[theorem]{Proposition}
\theoremstyle{definition}

\allowdisplaybreaks

\begin{document} 
\maketitle

\begin{abstract}
Minimax optimal convergence rates for numerous classes of stochastic convex optimization problems are well characterized, where the
majority of results utilize iterate averaged stochastic gradient descent (SGD) with polynomially decaying step sizes. In contrast, the behavior of SGD’s final iterate has received much less attention despite the widespread use in practice. Motivated by this observation, this work provides a detailed study of the following question: what rate is achievable using the final iterate of SGD for the streaming least squares regression problem with and without strong convexity? 

First, this work shows that even if the time horizon $T$ (i.e. the number of iterations that SGD is run for) is known in advance, the behavior of SGD’s final iterate with \emph{any} polynomially decaying learning rate scheme is highly sub-optimal compared to the statistical minimax rate (by a condition number factor in the strongly convex case and a factor of $\sqrt{T}$ in the non-strongly convex case). In contrast, this paper shows that \emph{Step Decay} schedules, which cut the learning rate by a constant factor every constant number of epochs (i.e., the learning rate decays geometrically) offer significant improvements over \emph{any} polynomially decaying step size schedule. In particular, the behavior of the final iterate with step decay schedules is off from the statistical minimax rate by only \emph{log} factors (in the condition number for the strongly convex case, and in $T$ in the non-strongly convex case). Finally, in stark contrast to the known horizon case, this paper shows that the anytime (i.e. the limiting) behavior of SGD’s final iterate is poor (in that it queries iterates with highly sub-optimal function value infinitely often, i.e. in a limsup sense) irrespective of the stepsize scheme employed. These results demonstrate the subtlety in establishing optimal learning rate schedules (for the final iterate) for stochastic gradient procedures in fixed time horizon settings.

\end{abstract}

\section{Introduction}

Large scale machine learning relies almost exclusively on stochastic optimization methods~\citep{BottouB07}, which include stochastic gradient descent (SGD)~\citep{RobbinsM51} and its variants~\cite{Duchi2011Adagrad, JohnsonZ13}.  In this work, we restrict our attention to the SGD algorithm where we are concerned with the behavior of the final iterate (i.e. the last point when we terminate the algorithm). A majority of (minimax optimal) theoretical results for SGD focus on polynomially decaying stepsizes~\citep{dekel12,RakhlinSS12,JulienSB12,Bubeck14} (or constant stepsizes~\citep{BachM13,defossez15,jain2016parallelizing} for the case of least squares regression) coupled with iterate averaging \citep{Ruppert88, PolyakJ92} to achieve minimax optimal rates of convergence. However, practical SGD implementations typically return the final iterate of a stochastic gradient procedure. This line of work in theory (based on iterate averaging) and its discrepancy with regards to practice leads to the question with regards to the behavior of SGD's final iterate. Indeed, this question has motivated several efforts in stochastic convex optimization literature as elaborated below.

{\bf Non-Smooth Stochastic Optimization: }The work of~\cite{Shamir12} raised the question with regards to the behavior of SGD's final iterate for non-smooth stochastic optimization (with/without strong convexity). The work of~\cite{ShamirZ12} answered this question, indicating that SGD's final iterate with polynomially decaying stepsizes achieves near minimax rates (up to $\log$ factors) in an anytime (i.e. in a limiting) sense (when number of iterations SGD is run for is not known in advance). Under specific choices of step size sequences, \cite{ShamirZ12}'s result on SGD's final iterate is tight owing to the recent work of~\citet{HarveyLPR18}. More recently~\cite{JainNN19} presented an approach indicating that a more nuanced stepsize sequence serves to achieve minimax rates (up to constant factors) for the non-smooth stochastic optimization setting when the end time $T$ is known in advance. 
\begin{figure}[t!]\vspace*{-14pt}
	\begin{minipage}{.45\linewidth}
		\RestyleAlgo{boxruled}
		\begin{algorithm}[H]
			\caption{Step Decay scheme}
			\label{algo:StepDecay}
			\ \KwIn{Initial vector $\w$, starting learning rate $\eta_0$, number of iterations $T$}
			\ \KwOut{$\w$}
			\ \For{$\ell \gets 1$ \KwTo $\log T$}{\vspace{1mm}
				$\eta_{\ell} \leftarrow {\eta_0}/{2^\ell}$ \vspace{1.5mm}\\
				\For{$t \gets 1$ \KwTo ${T}/{\log T}$}{
					$\w \leftarrow \w - \eta_\ell \widehat{\nabla}f(\w)$
				}\ 
			}
		\end{algorithm}
	\end{minipage}\hspace*{5pt}
	\begin{minipage}{.55\linewidth}
		\includegraphics[width=\columnwidth]{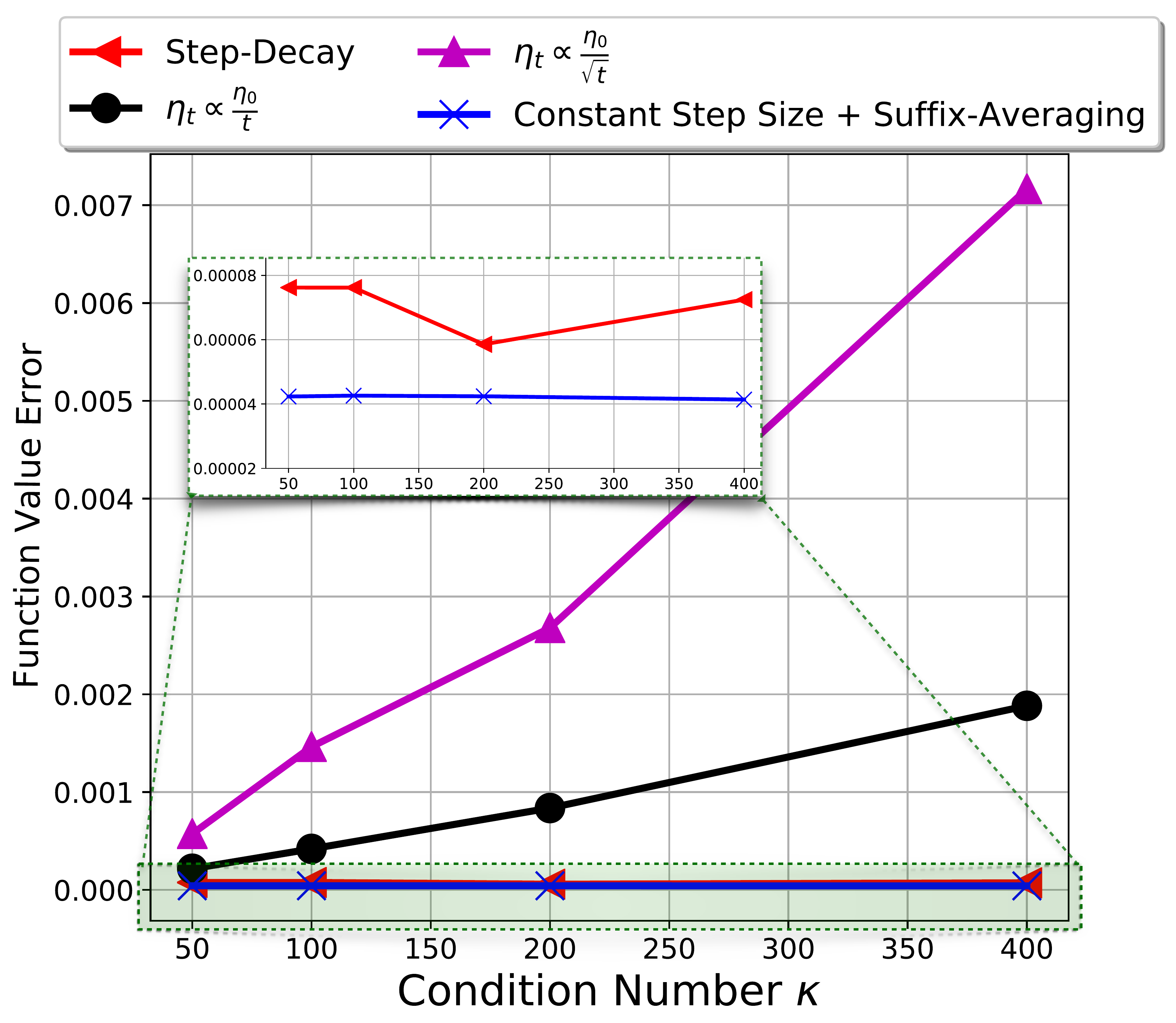}
	\end{minipage}
	\caption{(Left) The Step Decay scheme for stochastic gradient descent. Note that the algorithm requires just two parameters - the starting learning rate $\eta_0$ and number of iterations $T$.\\ (Right) Plot of function value error vs. condition number for the final iterate of polynomially decaying stepsizes i.e., equation(\ref{eq:p1},\ref{eq:p2}), step-decay schedule (Algorithm~\ref{algo:StepDecay}) compared against the minimax optimal suffix averaged iterate with a constant stepsize \citep{jain2016parallelizing} for a synthetic two-dimensional least squares regression problem(\ref{eq:objectiveFn}). The condition number $\kappa$ is varied as $\{50,100,200,400\}$. Exhaustive grid search is performed on starting stepsize and decay parameters. Initial excess risk is ${d\sigma^2}$ and the algorithm is run for $T=\kappa_{\max}^2=400^2$ steps (for all experiments); results are averaged over $5$ random seeds. Observe that the final iterate's error grows linearly as a function of the condition number $\kappa$ for the polynomially decaying stepsize schemes, whereas, the error does not grow as a function of $\kappa$ for the geometric ``step-decay'' stepsize scheme. See section~\ref{sapp:syn2d} in  the appendix  for details.}\label{fig:synthetic}
\end{figure}

{\bf Least Squares Regression (LSR):} In contrast to the non-smooth setting, the state of our understanding of SGD's final iterate for smooth stochastic convex optimization, or, say, the streaming least squares regression setting is far less mature $-$ this gap motivates our paper's contributions. In particular, this paper studies SGD's final iterate behavior under various stepsize choices for least squares regression (with and without strong convexity). The use of SGD's final iterate for the least mean squares objective has featured in several efforts~\citep{WidrowH60, Proakis74, WidrowS85, RoyS90}, but these results {\em do not} achieve minimax rates of convergence, which leads to the following question:

{\bf ``} Can polynomially decaying stepsizes (known to achieve minimax rates when coupled with iterate averaging~\citep{Ruppert88, PolyakJ92}) offer minimax optimal rates on SGD's {\em final} iterate when optimizing the streaming least squares regression objective? If not, is there {\em any} other family of stepsizes that can guarantee minimax rates on the final iterate of stochastic gradient descent? {\bf ''}

This paper presents progress on answering the above question $-$ refer to contributions below for more details. Note that the oracle model employed by this work (to quantify SGD's final iterate behavior) has featured in a string of recent results that present a non-asymptotic understanding of SGD for least squares regression, with the caveat being that these results crucially rely on {\em iterate averaging} with constant stepsize sequences~\citep{BachM13, defossez15, jain2016parallelizing, jain2017accelerating, JainKKNPS17, NeuR18}.

\begin{table*}[t]
	\begin{center}\captionsetup{singlelinecheck=off}
		\begin{adjustbox}{max width=\textwidth}
			\begin{tabular}{| c | c | c | c | c |}
				\hline
				& Assumptions & Minimax rate & \begin{tabular}{@{}c@{}}
					Rate w/ Final iterate \\ using best poly-decay
				\end{tabular} &  \begin{tabular}{@{}c@{}} Rate w/ Final iterate  \\ using Step Decay  \end{tabular} \\ 
				\hline
				\begin{tabular}{@{}c@{}} General \\ convex functions \end{tabular} & \begin{tabular}{@{}c@{}} $\E{\norm{\widehat{\nabla} f}^2} \leq G^2$ \\ $\mathrm{Diam}\left(\mathrm{Constraint Set}\right) \leq D$ \end{tabular} & $\frac{GD}{\sqrt{T}} $ & \begin{tabular}{@{}c@{}} $\Theta\left( \frac{GD}{\sqrt{T}}\cdot \red{\log T} \right) $\\ \citep{ShamirZ12,HarveyLPR18}\end{tabular}& -- \\
				\hline
				\begin{tabular}{@{}c@{}} Non-strongly convex \\ least squares regression \end{tabular} & Eq.~\eqref{eq:varAtOpt} & $\frac{\sigma^2 d}{T}$ & \begin{tabular}{@{}c@{}} $\Omega\left( \frac{\sigma^2 d}{T} \cdot \red{\frac{\sqrt{T}}{\log T}} \right)$ \\ (This work - Theorem~\ref{thm:poly}) \end{tabular} & \begin{tabular}{@{}c@{}} $\order{ \frac{\sigma^2 d}{T} \cdot \red{\log T }}$ \\ (This work - Theorem~\ref{thm:const-cut}) \end{tabular} \\
				\hline
				\begin{tabular}{@{}c@{}} General strongly \\ convex functions \end{tabular} & \begin{tabular}{@{}c@{}} $\E{\norm{\widehat{\nabla} f}^2} \leq G^2$ \\ $\nabla^2 f \succeq \mu \eye$ \end{tabular} & $\frac{G^2}{\mu T} $ & \begin{tabular}{@{}c@{}}$\Theta\left( \frac{G^2}{ \mu T} \cdot \red{\log T} \right) $\\\citep{ShamirZ12,HarveyLPR18}\end{tabular} & -- \\
				\hline
				\begin{tabular}{@{}c@{}} Strongly convex \\ least squares regression\end{tabular} & \begin{tabular}{@{}c@{}} Eq.~\eqref{eq:varAtOpt} \\ $\nabla^2 f \succeq \mu \eye$ \end{tabular} & $\frac{\sigma^2 d}{T}$ & \begin{tabular}{@{}c@{}} $\Omega\left( \frac{\sigma^2 d}{T} \cdot \red{\kappa} \right)$ \\ ({This work} - Theorem~\ref{thm:poly}) \end{tabular} &\begin{tabular}{@{}c@{}} $\order{ \frac{\sigma^2 d}{T} \cdot \red{\log T}}$ \\ (This work - Theorem~\ref{thm:const-cut}) \end{tabular} \\
				\hline
			\end{tabular}
		\end{adjustbox}
		\caption[]{Comparison of sub-optimality for {\em final} iterate of SGD (i.e., $\E{f(\w_T)} - f(\w^*)$) for stochastic convex optimization problems. This paper's focus is on SGD's final iterate for streaming least squares regression. The minimax rate refers to the best possible worst case rate with access to stochastic gradients (typically achieved with iterate averaging methods~\citep{PolyakJ92,dekel12,RakhlinSS12}); the \textcolor{red}{red} shows the multiplicative factor increase (over the minimax rate) using the final iterate, under two different learning rate schedules - the polynomial decay and the step decay (refer to Algorithm~\ref{algo:StepDecay}). Polynomial decay schedules are of the form  $\eta_t\propto1/t^\alpha$ (for appropriate $\alpha\in[0.5,1]$). For the general convex cases below, the final iterate with a polynomial decay scheme is off minimax rates by a $\log{T}$ factor (in an anytime/limiting sense)~\citep{ShamirZ12}. Here $\widehat{\nabla}f, \nabla f = \E{\widehat{\nabla} f}, \nabla^2 f$ denotes the stochastic gradient, gradient and the Hessian of the function $f$. With regards to least squares, we assume equation~\eqref{eq:varAtOpt}, following recent efforts~\citet{BachM13, defossez15, jain2016parallelizing}. While polynomially decaying stepsizes are nearly minimax optimal for general (strongly) convex functions, this paper indicates they are highly suboptimal on the final iterate for least squares. The geometrically decaying Step Decay schedule (Algorithm~\ref{algo:StepDecay}) provides marked improvements over any polynomial decay scheme on the final iterate for least squares. For simplicity of presentation, the results for least squares regression do not show dependence on initial error. See Theorems~\ref{thm:poly} and~\ref{thm:const-cut} for precise statements (and \cite{Nemirovsky83, ShamirZ12, HarveyLPR18} for precise statements of the general case).}\label{tab:main1}

	\end{center}
\end{table*}
\paragraph{Our contributions:}
This work establishes upper and lower bounds on the behavior of SGD's final iterate, as run with polynomially decaying stepsizes as well as {\em step decay} schedules which tends to cut the stepsize by a constant factor after every constant number of epochs (see algorithm~\ref{algo:StepDecay}), by considering the streaming least squares regression problem (with and without strong convexity). Our main result indicates that step decay schedules offer significant improvements in achieving near minimax rates over polynomially decaying stepsizes in the known horizon case (when the end time $T$ is known in advance). Figure~\ref{fig:synthetic} illustrates that this difference is evident (empirically) even when optimizing a two-dimensional synthetic least squares objective. Table~\ref{tab:main1} provides a summary. Finally, we present results that indicate the subtle (yet significant) differences between the known time horizon case and the anytime (i.e. the limiting) behavior of SGD's final iterate (see below). Note that proofs of our main claims can be found in the appendix.

Our main contributions are as follows:\vspace*{-2mm}
\begin{itemize}
	\item {\em Sub-optimality of polynomially decaying stepsizes:} For the strongly convex least squares case, this work shows that the final iterate of a polynomially decaying stepsize scheme (i.e. with $\eta_t\propto 1/t^\alpha$, with $\alpha\in[0.5,1]$) is off the minimax rate $d\sigma^2/T$ by a factor of the {\em condition number} of the problem. For the non-strongly convex case of least squares, we show that {\em any} polynomially decaying stepsize can achieve a rate no better than $d\sigma^2/\sqrt{T}$ (up to $\log$ factors), while the minimax rate is $d\sigma^2/T$.\vspace*{-0.12em}
	\item {\em Near-optimality of the step-decay scheme:} Given a fixed end time $T$, the step-decay scheme (algorithm~\ref{algo:StepDecay}) presents a final iterate that is off the statistical minimax rate by just a $\log(T)$ factor when optimizing the strongly convex and non-strongly convex least squares regression \footnote{This dependence can be improved to $\log$ of the condition number of the problem (for the strongly convex case) using a more refined stepsize decay scheme.}, thus indicating vast improvements over polynomially decaying stepsize schedules. We note here that our Theorem~\ref{thm:const-cut} for the non-strongly case offers a rate on the initial error (i.e., the bias term) that is off the best known rate~\citep{BachM13} (that employs iterate averaging) by a dimension factor. That said, Algorithm~\ref{algo:StepDecay} is rather straightforward and employs the knowledge of just an initial learning rate and number of iterations for its implementation.
	\item {\em SGD has to query bad iterates infinitely often:} For the case of optimizing strongly convex least squares regression, this work shows that any stochastic gradient procedure (in a $\limsup$ sense) {\em must} query sub-optimal iterates (off by nearly a condition number) infinitely often.
	\item Complementary to our theoretical results for the stochastic linear regression, we evaluate the empirical performance of different learning rate schemes when training a residual network on the~\cifar~ dataset and observe that the continuous variant of step decay schemes (i.e. an exponential decay) indeed compares favorably to polynomially decaying stepsizes.
\end{itemize}
While the upper bounds established in this paper (section~\ref{sec:const-cut}) merit extensions towards broader smooth convex functions (with/without strong convexity), the lower bounds established in sections~\ref{sec:poly}, \ref{sec:inf-horizon} present implications towards classes of smooth stochastic convex optimization. Even in terms of upper bounds, note that there are fewer results on non-asymptotic behavior of SGD (beyond least squares) when working in the oracle model considered in this work (see below). \citet{BachM11,BachM13,Bach14,NeedellSW16} are exceptions, yet they do not achieve minimax rates on appropriate problem classes;~\citet{FrostigGKS15} does not work in standard stochastic first order oracle model~\citep{Nemirovsky83,AgarwalBRW12}, so their work is not directly comparable to examine extensions towards broader function classes.

As a final note, this paper's result on the sub-optimality of standard polynomially decaying stepsizes for classes of smooth and strongly convex optimization doesn't contradict the (minimax) optimality results in stochastic approximation~\citep{PolyakJ92}. Iterate averaging coupled with polynomially decaying learning rates (or constant learning rates for least squares~\citep{BachM13, defossez15, jain2016parallelizing}) does achieve minimax rates~\citep{Ruppert88, PolyakJ92}. However, as mentioned previously, this work deals with SGD's final iterate behavior (i.e. without iterate averaging), since this bears more relevance towards practice.
{\bf Related work:}
\cite{RobbinsM51} introduced the stochastic approximation problem and Stochastic Gradient Descent (SGD). They present conditions on stepsize schemes satisfied by asymptotically convergent algorithms: these schemes are referred to as ``convergent'' stepsize sequences. ~\cite{Ruppert88,PolyakJ92} proved the asymptotic optimality of iterate averaged SGD with larger stepsize sequences. In terms of oracle models and notions of optimality, there exists two lines of thought (see also~\cite{jain2017accelerating}):

{\em Towards statistically optimal estimation procedures:} The goal of this line of thought is to match the excess risk of the statistically optimal estimator~\citep{anbar1971optimal, KushnerClark, PolyakJ92, lehmann1998theory} on every problem instance. Several efforts consider SGD in this oracle~\citep{BachM11, Bach14, DieuleveutB15, FrostigGKS15, NeedellSW16} presenting non-asymptotic results, often with iterate averaging. With regards to least squares, \citet{BachM13, defossez15, FrostigGKS15,jain2016parallelizing, jain2017accelerating,NeuR18} use constant step-size SGD with iterate averaging to achieve minimax rates (on a per-problem basis) in this oracle model. SGD's final iterate behavior for least squares has featured in several efforts in the signal processing/controls literature~\citep{WidrowH60, NagumoN67, Proakis74, WidrowS85, RoyS90, SharmaSB98}, without achieving minimax rates. This paper works in this oracle model and analyzes SGD's final iterate behavior with various stepsize choices.

{\em Towards optimality under bounded noise assumptions:} The other line of thought presents algorithms with access to stochastic gradients satisfying bounded noise assumptions, aiming to match lower bounds provided in~\cite{Nemirovsky83, RaginskyR11, AgarwalBRW12}. Asymptotic properties of ``convergent'' stepsize schemes have been studied in great detail \citep{KushnerClark, Benveniste90, ljungPW92, BharathB99, KushnerY03, Lai03, Borkar08}. ~\cite{dekel12, JulienSB12, RakhlinSS12, ghadimi2012optimal, ghadimi2013optimal, HazanK14, Bubeck14, DieuleveutFB16} use iterate averaged SGD to achieve minimax rates for various problem classes non-asymptotically. ~\citet{AllenZhu18} present an alternative approach towards minimizing the gradient norm with access to stochastic gradients. As noted,~\cite{ShamirZ12} achieves anytime optimal rates (upto a $\log{T}$ factor) with the final iterate of an SGD procedure, and this is shown to be tight with the recent work of~\citet{HarveyLPR18}. \citet{JainNN19} achieve minimax rates on the final iterate using a nuanced stepsize scheme when the number of iterations is fixed in advance. 

{\em Geometrically Decaying Stepsize Schedules} date to~\citet{Goffin:subgrad}.~\citet{DavisD19} employ the stepdecay schedule to prove high-probability guarantees for SGD with strongly convex objectives. In stochastic optimization, several other works, including~\citet{GhadimiL13, HazanK14, AybatFGO19, KulunchakovM19} consider doubling argument based approaches, where the epoch length is doubled everytime the stepsizes are halved. The step decay schedule is employed to yield faster rates of convergence under certain growth (and related) conditions both in convex~\citep{XuLY16} and non-convex settings~\citep{YangYYJ18,DavisDC19}.

{\bf Paper organization:} 
Section~\ref{sec:setup} describes notation and problem setup. Section~\ref{sec:comparison} presents our results on the sub-optimality of polynomial decay schemes and the near optimality of the step decay scheme. Section~\ref{sec:inf-horizon} presents results on the anytime behavior of SGD (i.e. the asymptotic/infinite horizon case). Section~\ref{sec:exp} presents experimental results and Section~\ref{sec:conc} presents conclusions.


\section{Problem Setup}\label{sec:setup}
\textbf{Notation}: We present the setup and associated notation in this section. We represent scalars with normal font $a, b, L$ etc., vectors with boldface lowercase characters $\a,\bb$ etc. and matrices with boldface uppercase characters $\A,\B$ etc. We represent positive semidefinite (PSD) ordering between two matrices using $\succeq$. The symbol $\gtrsim$ represents that the inequality holds for some universal constant. 

We consider here the minimization of the following expected square loss objective:
\begin{align}\label{eq:objectiveFn}
\min_{\w} f(\w) \, \textrm{ where } f(\w) \eqdef \tfrac{1}{2} \mathbb{E}_{(\x,y)\sim\mathcal{D}}[(y-\langle\w,\x\rangle)^2].
\end{align}
Note that the hessian of the objective $\Cov\eqdef\nabla^2 f(\w)=\E{\x\x^\top}$. We are provided access to stochastic gradients obtained by sampling a new example $(\x_t,y_t)\sim\mathcal{D}$. These examples satisfy:
\begin{align*}
y = \langle \ws,\x\rangle +\epsilon,
\end{align*}
where, $\epsilon$ is the noise on the example pair $(\x,y)\sim\mathcal{D}$ and $\ws$ is a minimizer of the objective $f(\w)$. Given an initial iterate $\w_0$ and stepsize sequence $\{\eta_t\}$, our stochastic gradient update is:
\begin{align}
\w_{t+1}\leftarrow\w_t-\eta_t\widehat{\nabla}f(\w_{t-1});\ \ \ \ \widehat{\nabla}f(\w_t) = -(y_t - \langle \w_t,\x_t\rangle)\cdot \x_t\label{eq:sgOracle}.
\end{align}
We assume that the noise $\epsilon = y-\langle \ws,\x\rangle\ \ \forall \ \ (\x,y)\sim\mathcal{D}$ satisfies the following condition:
\begin{align}\label{eq:varAtOpt}
\Sigma \eqdef \E{\widehat{\nabla}f(\ws) \widehat{\nabla}f(\ws)^\top} = \mathbb{E}_{(\x,y)\sim\mathcal{D}}[(y-\langle\ws,\x\rangle)^2\x\x^\top]\preceq\sigma^2\Cov.
\end{align}
Next, assume that covariates $\x$ satisfy the following fourth moment inequality:
\begin{align}\label{eq:fourthmoment}
\E{\norm{\x}^2 \x\x^\top} \preceq \infbound\ \Cov
\end{align}
This assumption is satisfied, say, when the norm of the covariates $\sup \norm{\x}^2 < \infbound$, but is true more generally. Finally, note that both~\ref{eq:varAtOpt} and~\ref{eq:fourthmoment} are general and are used in recent works~\citep{BachM13,jain2016parallelizing} that present a sharp analysis of SGD for streaming least squares problem. Next, we denote by
\begin{align*}
\mu \eqdef \lambda_{\textrm{min}}\left(\Cov\right), \quad
L \eqdef \lambda_{\textrm{max}}\left(\Cov\right), \mbox{  and },
\kappa \eqdef {\infbound}/{\mu}
\end{align*}
the smallest eigenvalue, largest eigenvalue and condition number of $\Cov$ respectively. $\mu > 0$ in the strongly convex case but not necessarily so in the non-strongly convex case (in section~\ref{sec:comparison} and beyond, the non-strongly case is referred to as the ``smooth'' case). 
Let $\ws \in \argmin_{\w \in \Rd} f(\w)$. The excess risk of an estimator $\w$ is $f(\w)-f(\ws)$. Given $t$ accesses to the stochastic gradient oracle in equation~\ref{eq:sgOracle}, any algorithm that uses these stochastic gradients and outputs $\widehat{\w}_t$ has sub-optimality that is lower bounded by $\frac{\sigma^2 d}{t}$. More concretely, we have that~\citep{van2000asymptotic}
\[
\lim_{t\rightarrow \infty} \frac{\E{f(\widehat{\w}_t)}-f(\ws)}{\sigma^2 d /t} \geq 1 \, .
\] The rate of $\left(1+o(1)\right)\cdot{\sigma^2 d}/{t}$ is achieved using iterate averaged SGD~\citep{Ruppert88,PolyakJ92} with constant stepsizes~\citep{BachM13,defossez15,jain2016parallelizing}. This rate of $\sigma^2 d /t$ is called the statistical minimax rate.

\section{Main results}\label{sec:comparison}

Sections~\ref{sec:poly},~\ref{sec:const-cut} consider the fixed time horizon setting; the former presents the significant sub-optimality of polynomially decaying stepsizes on SGD's final iterate behavior, the latter section presenting the near-optimality of SGD's final iterate. Section~\ref{sec:inf-horizon} presents negative results on SGD's final iterate behavior (irrespective of stepsizes employed), in the anytime (i.e. limiting) sense.

\subsection{Suboptimality of polynomial decay schemes}
\label{sec:poly}
This section begins by showing that there exist problem instances where polynomially decaying stepsizes considered stochastic approximation theory~\citep{RobbinsM51,PolyakJ92} i.e., those of the form $\frac{a}{b+t^{\alpha}}$, for any choice of $a,b>0$ and $\alpha\in[0.5,1]$ are significantly suboptimal (by a factor of the condition number of the problem, or by $\sqrt{T}$ in the smooth case) compared to the statistical minimax rate~\citep{KushnerClark}.
\begin{theorem}\label{thm:poly}
Under assumptions~\ref{eq:varAtOpt},~\ref{eq:fourthmoment}, there exists a class of problem instances where the following lower bounds on excess risk hold on SGD's final iterate with polynomially decaying stepsizes when given access to the oracle as written in equation~\ref{eq:sgOracle}.
  
\noindent\textbf{Strongly convex case}: Suppose $\mu > 0$. For any condition number $\kappa$, there exists a least squares problem instance with initial suboptimality $f(\w_0) - f(\ws) \le \sigma^2 d$ such that, for any $T \ge \kappa^{\frac{4}{3}}$, and for all $a,b\geq 0$ and $0.5 \leq \alpha \leq 1$, and for the learning rate scheme $\eta_t = \frac{a}{b+t^{\alpha}}$, we have
	\begin{align*}
		\E{f(\w_T)} -f(\ws) \geq \exp\left(-\frac{T}{\kappa\log T}\right)\left(f(\w_0)-f(\ws)\right) + \frac{\sigma^2 d}{{64}} \cdot \frac{\kappa}{T} .
	\end{align*}
	\textbf{Smooth case}: For any fixed $T> 1$, there exists a least squares problem instance 
	such that, for all $a,b\geq 0$ and $0.5 \leq \alpha \leq 1$, and for the learning rate scheme $\eta_t = \frac{a}{b+t^{\alpha}}$, we have
	\begin{align*}
	\E{f(\w_T)} -f(\ws) \geq \left( L \cdot \norm{\w_{0} - \ws}^2 + \sigma^2 d \right) \cdot \frac{1}{\sqrt{T} \log T}.
	\end{align*}
\end{theorem}
For both cases (with/without strong convexity), the minimax rate is $\sigma^2 d/T$. In the strongly convex case, SGD's final iterate with polynomially decaying stepsizes pays a suboptimality factor of $\Omega(\kappa)$, whereas, in the smooth case, SGD's final iterate pays a suboptimality factor of $\Omega\left(\frac{\sqrt{T}}{\log T}\right)$.

\subsection{Near optimality of Step Decay schemes}\label{sec:const-cut}
Given the knowledge of an end time $T$ when the algorithm is terminated, this section presents the step decay schedule (Algorithm~\ref{algo:StepDecay}), which offers significant improvements over standard polynomially decaying stepsize schemes, and obtains near minimax rates (off by only a $\log(T)$ factor).
\begin{theorem}\label{thm:const-cut}
Suppose we are given access to the stochastic gradient oracle~\ref{eq:sgOracle} satisfying Assumptions~\ref{eq:varAtOpt} and~\ref{eq:fourthmoment}. Running Algorithm~\ref{algo:StepDecay} with an initial stepsize of $\eta_1=1/(2\infbound)$ allows the algorithm to achieve the following excess risk guarantees.
\begin{itemize}
\item \textbf{Strongly convex case}:  Suppose $\mu > 0$. We have:
  \[
    \E{f(\w_T)} -f(\ws) \leq
    2\cdot\exp\left(-\frac{T}{2\kappa\log
        T\log\kappa}\right)\left(f(\w_0)-f(\ws)\right) + 4
    \sigma^2 d \cdot \frac{\log T}{T} \, . \]
\item \textbf{Smooth case}: We have:
  \[\E{f(\w_T)} -f(\ws) \leq 2\cdot\left(\infbound d \cdot \norm{\w_0-\ws}^2 +
      2{\sigma^2 d} \right) \cdot \frac{\log T}{T} \, \]
\end{itemize}
\end{theorem}
While theorem~\ref{thm:const-cut} presents significant improvements over polynomial decay schemes, as mentioned in the contributions, the above result presents a worse rate on the initial error (by a dimension factor) in the smooth case (i.e. non-strongly convex case), compared to the best known result~\citep{BachM13}, which relies heavily on iterate averaging to remove this factor. It is an open question with regards to whether this factor can actually be improved or not. Furthermore, comparing the initial error dependence between the lower bound for the smooth case (Theorem~\ref{thm:poly}) with the upper bound for the step decay scheme, we believe that the dependence on the smoothness $L$ should be improved to one on the $\infbound$. 

In terms of the variance, however, note that the polynomial decay schemes, are plagued by a polynomial dependence on the condition number $\kappa$ (for the strongly convex case), and are off the minimax rate by a $\sqrt{T}$ factor (for the smooth case). The step decay schedule, on the other hand, is off the minimax rate~\citep{Ruppert88,PolyakJ92,van2000asymptotic} by only a $\log(T)$ factor. It is worth noting that Algorithm~\ref{algo:StepDecay} admits an efficient implementation in that it requires the knowledge only of $\infbound$ (similar to iterate averaging results~\citep{BachM13,jain2016parallelizing}) and the end time $T$. Finally, note that this $\log T$ factor can be improved to a $\log \kappa$ factor for the strongly convex case by using an additional polynomial decay scheme before switching to the Step Decay scheme. 
\begin{proposition}\label{thm:compl}
Suppose we have access to the stochastic gradient oracle~\ref{eq:sgOracle} satisfying Assumptions~\ref{eq:varAtOpt} and~\ref{eq:fourthmoment}. Let $\mu>0$ and $\kappa \ge 2$. For any problem and fixed time horizon $T/\log T > 5 \kappa$, there exists a learning rate scheme that achieves
	\begin{align*}
	\E{f(\w_T)} - f(\ws) \leq 2\exp(-T/(6\kappa\log\kappa))\cdot(f(\w_0)-f(\ws)) + {100\log_2 \kappa}{} \cdot \frac{\sigma^2 d}{T}.
	\end{align*}
\end{proposition}
In order to have improved the dependence on the variance from $\log(T)$ (in theorem~\ref{thm:const-cut}) to $\log(\kappa)$ (in proposition~\ref{thm:compl}), we require access to the strong convexity parameter $\mu=\lamminH$ in addition to $\infbound$ and knowledge of the end time $T$. This parallels results known for general strongly convex setting~\citep{RakhlinSS12,JulienSB12,ShamirZ12,Bubeck14,JainNN19}.

As a final remark, note that this section's results (on step decay schemes) assumed the knowledge of a fixed time horizon $T$. In contrast, most results SGD's averaged iterate obtain anytime (i.e., limiting/infinite horizon) guarantees. Can we hope to achieve such guarantees with the final iterate?

\subsection{SGD queries bad points infinitely often}\label{sec:inf-horizon}
This section shows that obtaining near minimax rates with the {\em final} iterate is not possible without knowledge of the time horizon $T$. Concretely, we show that irrespective of the learning rates employed (be it polynomially decaying or step-decay), SGD {\em requires} to query a point with sub-optimality $\Omega(\kappa/\log \kappa) \cdot \sigma^2 d/T$ for infinitely many time steps $T$. 
\begin{theorem}\label{thm:limsup}
Suppose we have access to a stochastic gradient oracle~\ref{eq:sgOracle} satisfying Assumption~\ref{eq:varAtOpt},~\ref{eq:fourthmoment}. There exists a universal constant $C > 0$, and a problem instance such that SGD algorithm with any $\eta_t \le 1/2\infbound$ for all $t$\footnote{Learning rate more than $2/\infbound$ will make the algorithm diverge.}, we have
\begin{align*}
  \limsup_{T\rightarrow \infty} \frac{\E{f(\w_T)}-f(\ws)}{\left(\sigma^2 d/T\right)} \geq C \frac{\kappa }{\log (\kappa+1)}
  \, .
\end{align*}
\end{theorem}
\newcommand{\nt}{\tilde{T}}
\newcommand{\dt}{\Delta}
The bad points guaranteed to exist by Theorem~\ref{thm:limsup} are not rare. We show that such points occur at least once in $\order{\frac{\kappa}{\log \kappa}}$ iterations. Refer to Theorem~\ref{thm:badfraction} in appendix~\ref{app:inf-horizon} in the appendix.


\section{Experimental Results}\label{sec:exp}
We present experimental validation on the suitability of the Step-decay schedule (or more precisely, its continuous counterpart, which is the exponentially decaying schedule), and compare its with the polynomially decaying stepsize schedules. In particular, we consider the use of:\\
\noindent\begin{minipage}{0.33\linewidth}
	\begin{equation}
	\eta_t = \frac{\eta_0}{1+b\cdot t}\label{eq:p1}
	\end{equation}
\end{minipage}
\begin{minipage}{0.33\linewidth}
	\begin{equation}
	\eta_t = \frac{\eta_0}{1+b\sqrt{t}}\label{eq:p2}
	\end{equation}
\end{minipage}
\begin{minipage}{0.33\linewidth}
	\begin{equation}
	\vspace*{1.3em}
	\eta_t = \eta_0\cdot\exp{(-b\cdot t)}\label{eq:e1}.
	\end{equation}
\end{minipage}
Where, we perform a systematic grid search on the parameters $\eta_0$ and $b$. In the section below, we consider a real world non-convex optimization problem of training a residual network on the \cifar~dataset, with an aim to illustrate the practical implications of the results described in the paper. Refer to Appendix~\ref{app:exp} for more details.

\subsection{Non-Convex Optimization: Training a Residual Net on \cifar~}
\label{ssec:cifar10}
We consider training a $44-$layer deep residual network~\citep{kaiminghe2016resnet} with pre-activation blocks~\citep{kaiminghe2016identity} (dubbed preresnet-44) on \cifar~dataset. The code for implementing the network can be found here~\footnote{\url{https://github.com/D-X-Y/ResNeXt-DenseNet}}. For all experiments, we use Nesterov's momentum~\citep{Nesterov83} implemented in pytorch \footnote{\url{https://github.com/pytorch}} with a momentum of $0.9$, batchsize $128$, $100$ training epochs, $\ell_2$ regularization of $0.0005$. 

Our experiments are based on grid searching for the best learning rate decay scheme on the parametric family of learning rate schemes described above~\eqref{eq:p1},\eqref{eq:p2},\eqref{eq:e1}; all grid searches are performed on a separate validation set (obtained by setting aside one-tenth of the training dataset) and with models trained on the remaining $45000$ samples. For presenting the final numbers in the plots/tables, we employ the best hyperparameters from the validation stage and train it on the entire $50,000$ samples and average results run with $10$ different random seeds. The parameters for grid searches and other details are presented in Appendix~\ref{app:exp}. Furthermore, we always extend the grid so that the best performing grid search parameter lies in the interior of our grid search.

\textbf{How does the step decay scheme compare with the polynomially decaying stepsizes?} Figure~\ref{fig:cifar} and Table~\ref{table:cifar} present a comparison of the performance of the three schemes~\eqref{eq:p1}-\eqref{eq:e1}. These results demonstrate that the exponential scheme convicingly outperforms the polynomial step-size schemes.
\begin{figure}[H]
	\centering
	\subfigure{\includegraphics[width=0.37\columnwidth]{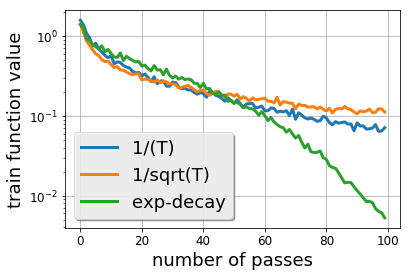}}
	\subfigure{\includegraphics[width=0.37\columnwidth]{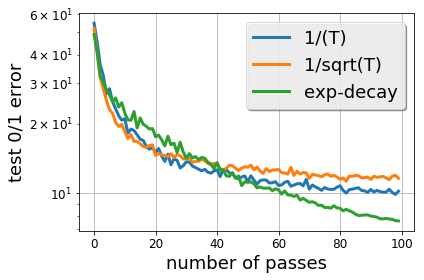}}\vspace*{-10pt}
	\caption{Plot of the training function value (left) and test $0/1-$ error (right) comparing the three decay schemes (two polynomial)~\ref{eq:p1},~\ref{eq:p2}, (and one exponential)~\ref{eq:e1} scheme.}
	\label{fig:cifar}
\end{figure}
\begin{table}[H]
	\centering
\begin{tabular}{|c | c | c |}
		\hline
		Decay Scheme & Train Function Value & Test $0/1$ error \\\hline\hline
		$O(1/t)$ (equation~\eqref{eq:p1}) & $0.0713\pm0.015$ & $10.20\pm0.7\%$ \\ \hline
		$O(1/\sqrt{t})$ (equation~\eqref{eq:p2})  & $0.1119\pm 0.036$ & $11.6\pm0.67\%$ \\\hline
		$\exp(-t)$ (equation~\eqref{eq:e1}) & {$\bf 0.0053\pm 0.0015$} & $\bf 7.58\pm0.21\%$ \\\hline
	\end{tabular}
	\caption{Comparing Train Cross-Entropy and Test $0/1$ Error of various learning rate decay schemes for the classification task on \cifar~ using a $44-$layer residual net with pre-activations.}
    \label{table:cifar}
\end{table}

\textbf{Does suffix  iterate averaging improve over final iterate's behavior for polynomially decaying stepsizes?} Towards answering this question, firstly, we consider the best performing values of equation~\ref{eq:p1} and~\ref{eq:p2}, and then, average iterates of the algorithm starting from $5, 10, 20, 40, 80, 85, 90, 95, 99$ epochs when training the model for a total of $100$ epochs. While such iterate averaging (and their suffix) variants have strong theoretical support for (stochastic) convex optimization~\citep{Ruppert88, PolyakJ92, RakhlinSS12, Bubeck14, jain2016parallelizing}, their impact on non-convex optimization is largely debatable. Nevertheless, this experiments's results (figure~\ref{fig:suffixAvg}) indicates that suffix averaging tends to hurt the algorithm's generalization behavior (which is unsurprising given the non-convex nature of the objective). Note that, figure~\ref{fig:suffixAvg} serves to indicate that averaging the final few ($\le 5$) epochs tends to offer nearly the same result as the final iterate's behavior, indicating that the gains of using suffix iterate averaging are relatively limited for several such settings.
\begin{figure}[H]
	\centering
	\subfigure{\includegraphics[width=0.37\columnwidth]{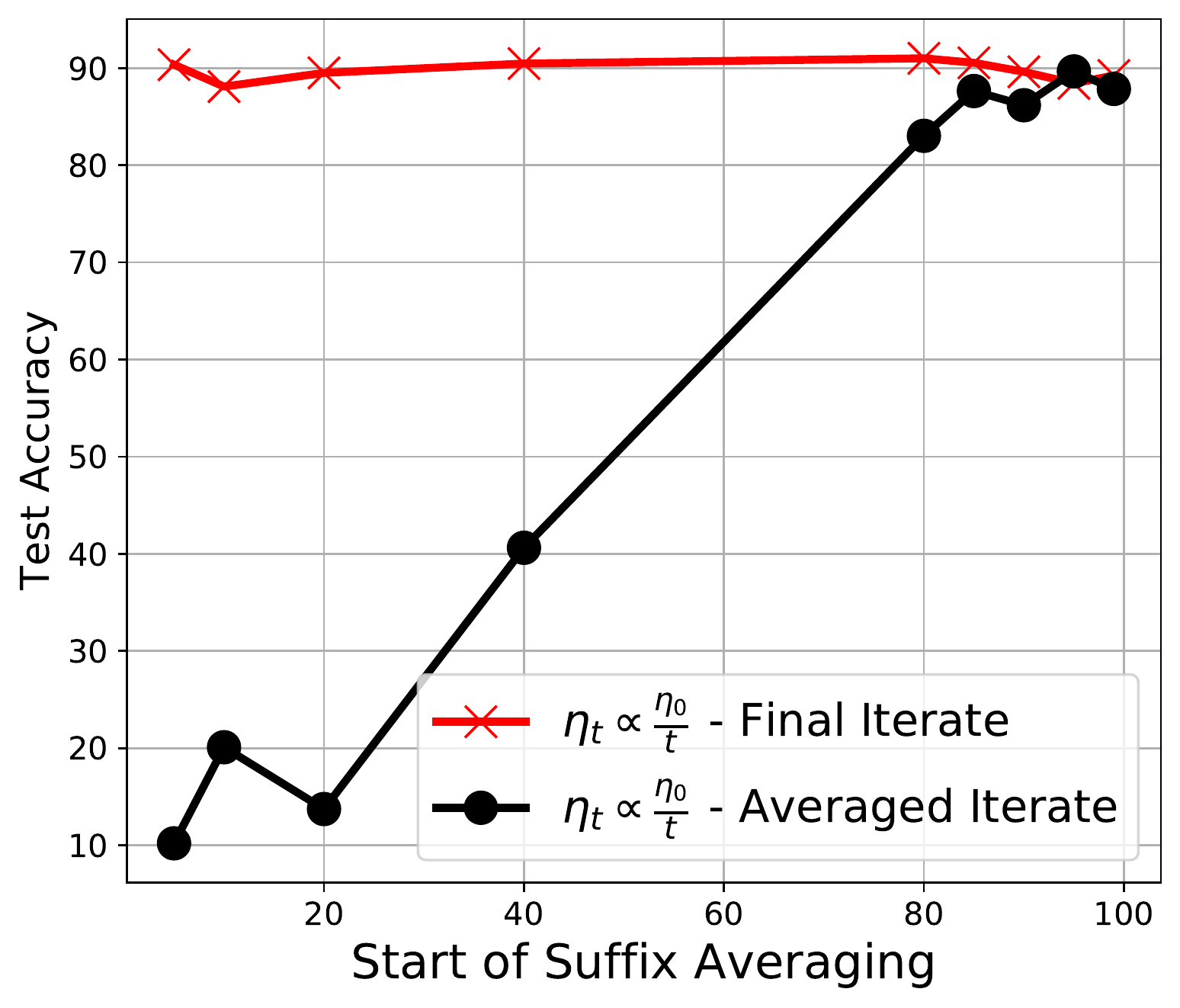}}
	\subfigure{\includegraphics[width=0.37\columnwidth]{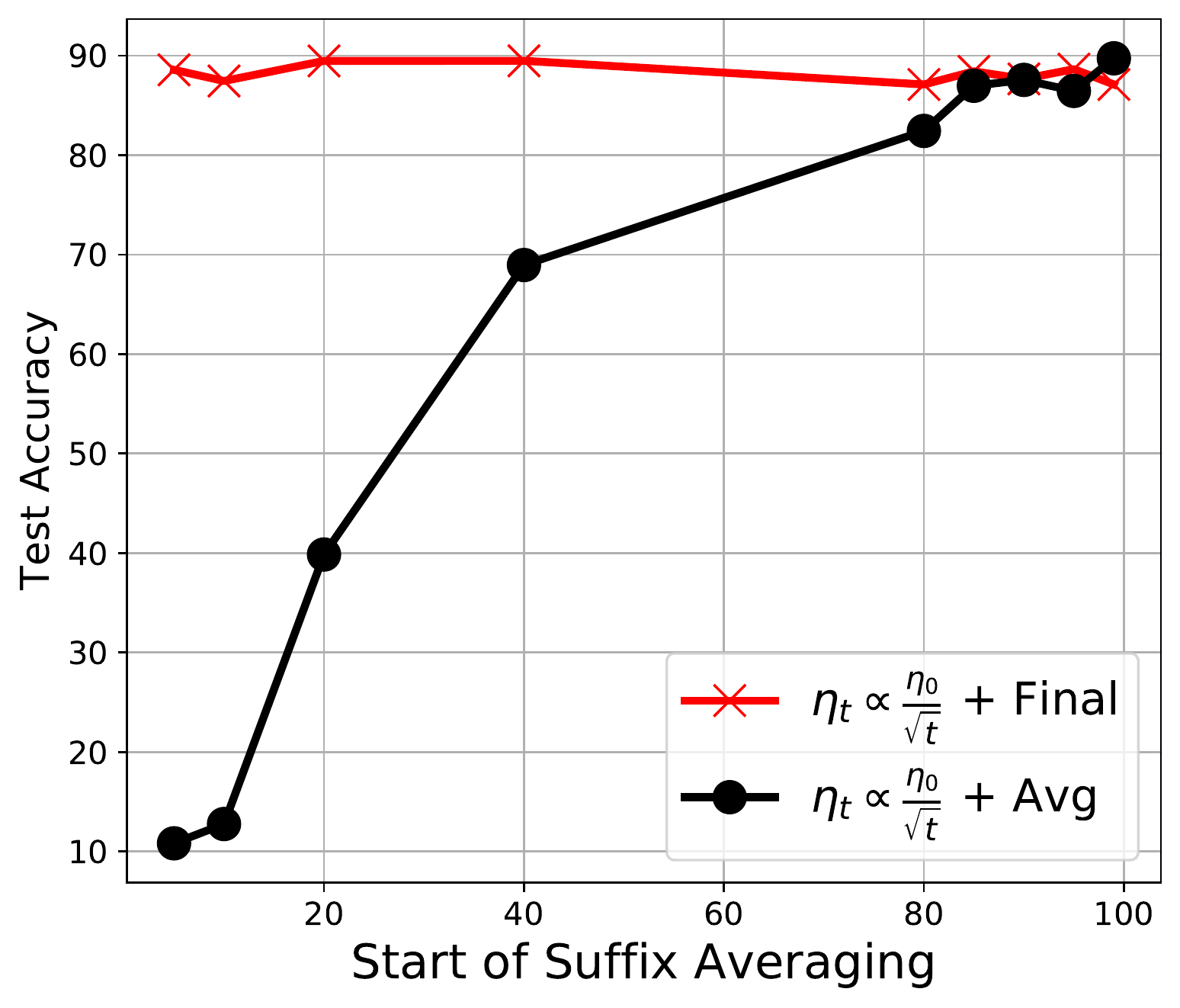}}\vspace*{-10pt}
	\caption{Performance of the suffix averaged iterate compared to the final iterate when varying the iteration when iterate averaging is begun from $\{5, 10, 20, 40, 80, 85, 90, 95, 99\}$ epochs for the $1/T$ learning rate~\ref{eq:p1} (left) and the $1/\sqrt{T}$ learning rate~\ref{eq:p2} (right).}
	\label{fig:suffixAvg}
\end{figure}

\textbf{Does our result on ``knowing'' the time horizon (for step-decay schedule) present implications towards hyper-parameter search methods that work based on results from truncated runs?} Towards answering this question, consider the figure~\ref{fig:cifar2} and Tables~\ref{table:cifarExpCompareTrain} and~\ref{table:cifarExpCompareVal}, which present a comparison of the performance of three exponential decay schemes each of which is tuned to achieve the best performance at $33$, $66$ and $100$ epochs respectively. The key point to note is that best performing hyperparameters at $33$ and $66$ epochs are not the best performing at $100$ epochs (which is made stark from the perspective of the validation error - refer to table~\ref{table:cifarExpCompareVal}). This demonstrates that hyper parameter selection methods that tend to discard hyper-parameters which don't perform well at earlier stages of the optimization (i.e. based on comparing results on truncated runs), which, for e.g., is indeed the case with hyperband~\citep{li2017hyperband}, will benefit from a round of rethinking.
\begin{figure}[H]
	\centering
	\subfigure{\includegraphics[width=0.37\columnwidth]{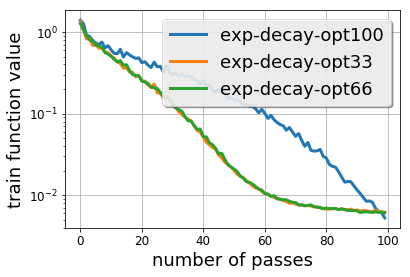}}
	\subfigure{\includegraphics[width=0.37\columnwidth]{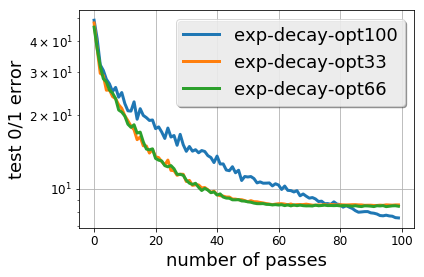}}\vspace*{-10pt}
	\caption{Plot of the training function value (left) and test $0/1-$ error (right) comparing exponential decay scheme (equation~\ref{eq:e1}), with parameters optimized for $33$, $66$ and $100$ epochs.}
	\label{fig:cifar2}\vspace*{-5pt}
\end{figure}\vspace*{-0.5em}
\begin{table}[H]
	\centering
\begin{adjustbox}{max width=\textwidth}
		\begin{tabular}{|c | c | c | c|}
			\hline
			Decay Scheme & Train FVal $@33$  & Train FVal $@66$ & Train FVal $@100$ \\\hline\hline
			$\exp(-t)$ [optimized for $33$ epochs] (eqn~\eqref{eq:e1}) & {$\bf 0.098\pm 0.006$} & {$\bf 0.0086\pm0.002$} & {$\bf 0.0062\pm0.0015$} \\ \hline
			$\exp(-t)$ [optimized for $66$ epochs] (eqn~\eqref{eq:e1}) & {$\bf 0.107\pm0.012$} & {$\bf 0.0088\pm0.0014$} & {$\bf 0.0061\pm0.0011$}\\\hline		
			$\exp(-t)$ [optimized for $100$ epochs] (eqn~\eqref{eq:e1}) & {$0.3\pm 0.06$} & $0.071\pm0.017$ & {$\bf 0.0053\pm0.0016$} \\\hline		
		\end{tabular}
	\end{adjustbox}
	\caption{Comparing training (softmax) function value by optimizing the exponential decay scheme with end times of $33/66/100$ epochs on \cifar~dataset using a $44-$layer residual net.}
\label{table:cifarExpCompareTrain}
\end{table}
\begin{table}[H]\vspace*{-4mm}
	\centering
\begin{adjustbox}{max width=\textwidth}
	\begin{tabular}{|c | c | c | c|}
		\hline
		Decay Scheme & Test $0/1$ $@33$  & Test $0/1$ $@66$ & Test $0/1$ $@100$ \\\hline\hline
		$\exp(-t)$ [optimized for $33$ epochs] (eqn~\eqref{eq:e1}) & {$\bf 10.36\pm0.235\%$} & {$ \bf 8.6\pm0.26\%$} & $8.57\pm0.25\%$ \\ \hline
		$\exp(-t)$ [optimized for $66$ epochs] (eqn~\eqref{eq:e1}) & {$\bf 10.51\pm0.45\%$} & {$\bf 8.51\pm0.13\%$} & $8.46\pm0.19\%$\\\hline		
		$\exp(-t)$ [optimized for $100$ epochs] (eqn~\eqref{eq:e1}) & {$14.42\pm1.47\%$} & $9.8\pm0.66\%$ & {$\bf7.58\pm0.21\%$} \\\hline		
	\end{tabular}
\end{adjustbox}
	\caption{Comparing test $0/1$ error by optimizing the exponential decay scheme with end times of $33/66/100$ epochs for the classification task on \cifar~dataset using a $44-$layer residual net.}
\label{table:cifarExpCompareVal}
\end{table}

\section{Conclusions and Discussion}\label{sec:conc}
The main contribution of this work shows that the behavior of SGD's final iterate for least squares regression is much more nuanced than what has been indicated by prior efforts that have primarily considered non-smooth stochastic convex optimization. The results of this paper point out the striking limitations of polynomially decaying stepsizes on SGD's final iterate, as well as sheds light on the effectiveness of starkly different schemes based on a Step Decay schedule. Somewhat coincidentally, practical implementations for certain classes of stochastic optimization do return the final iterate of SGD with step decay schedule $-$ this connection does merit an understanding through future work.

\paragraph{Acknowledgments:}
Rong Ge acknowledges funding from NSF CCF-$1704656$, NSF CCF-$1845171$ (CAREER), Sloan Fellowship and Google Faculty Research Award. Sham Kakade acknowledges funding from the Washington Research Foundation for Innovation in Data-intensive Discovery, NSF Award 1740551, and ONR award N$00014$-$18$-$1$-$2247$. Rahul Kidambi acknowledges funding from NSF Award $1740822$. 

\clearpage

\clearpage
\appendix
\section{Preliminaries}

Before presenting the lemmas establishing the behavior of SGD under various learning rate schemes, we introduce some notation. We recount that the SGD update rule denoted through:
\begin{align*}
\w_t = \w_{t-1} - \eta_t \widehat{\nabla f}(\w_{t-1})
\end{align*}
We then write out the expression for the stochastic gradient $\widehat{\nabla f}(\w_{t-1})$.
\begin{align*}
\widehat{\nabla f}(\w_{t-1}) = \x_t\x_t^\top (\w_{t-1}-\ws)-\epsilon_t \x_t,
\end{align*}
where, given the stochastic gradient corresponding to an example $(\x_t,y_t)\sim\mathcal{D}$, with $y_t = \iprod{\ws}{\x_t}+\epsilon_t$, the above stochastic gradient expression naturally follows. Now, in order to analyze the contraction properties of the SGD update rule, we require the following notation:
\begin{align*}
P_t = \eye - \eta_t \x_t\x_t^\top.
\end{align*}

\begin{lemma}\label{lem:biasVarTradeoffNew}
	[For e.g. Appendix A.2.2 from~\cite{jain2016parallelizing}] {\bf Bias-Variance tradeoff:} Running SGD for $T-$steps starting from $\w_{0}$ and a stepsize sequence $\{\eta_t\}_{t=1}^T$ presents a final iterate $\w_{T}$ whose excess risk is upper-bounded as:
	\begin{align*}
	\iprod{\Cov}{\E{(\w_T-\ws)\otimes(\w_T-\ws)}}&\leq 2\cdot \bigg(\iprod{\Cov}{\E{P_T\cdots P_1 (\w_{0}-\ws)\otimes(\w_{0}-\ws) P_1\cdots P_T}}\nonumber\\ &+ \iprod{\Cov}{\sum_{\tau=1}^{T} \eta_{\tau}^2\cdot \E{P_T\cdots P_{\tau+1} n_\tau\otimes n_\tau P_{\tau+1}\cdots P_T}} \bigg),
	\end{align*}
	where, $P_t = \eye-\eta_t\cdot \x_t\x_t^\top$ and $n_t = \epsilon_t x_t$. Note that $\E{n_t|\mathcal{F}_{t-1}}=0$ and $\E{n_t\otimes n_t | \mathcal{F}_{t-1}} \preceq \sigma^2 \Cov$, where, $\mathcal{F}_{t-1}$ is the filtration formed by all samples $(\x_1,y_1)\cdots(\x_{t-1},y_{t-1})$ until time $t$.
\end{lemma}

\begin{proof} One can view the contribution of the above two terms as ones stemming from SGD's updates, which can be written as:
\begin{align*}
\w_t &= \w_{t-1} - \eta_t \widehat{\nabla f}(w_{t-1})\nonumber\\
\w_t-\ws &= (\eye - \eta_t \x_t\x_t) (\w_{t-1}-\ws) + \eta_t n_t\nonumber\\
\w_t-\ws &= P_t\cdots P_1 (\w_0-\ws) + \sum_{\tau=1}^T P_t\cdots P_{\tau+1} \eta_\tau n_\tau\nonumber
\end{align*}
From the above equation, the result of the lemma follows straightforwardly. Now, clearly, if the noise $\epsilon$ and the inputs $\x$ are indepdent of each other, and if the noise is zero mean i.e. $\E{\epsilon}=0$, the above inequality holds with equality (without the factor of two). This is true more generally iff $$\E{\epsilon \x^{(i)}\x^{(j)}\x^{(k)}} =0.$$ For more details, refer to~\citet{defossez15}.


Now, in order to bound the total error, note that the original stochastic process associated with SGD's updates can be decoupled into two (simpler) processes, one being the noiseless process (which corresponds to reducing the dependence on the initial error, and is termed ``bias''), i.e., where, the recurrence evolves as:
\begin{align}
\label{eq:biasRecc}
\w_t^{\text{bias}}-\ws = P_t (\w_{t-1}^{\text{bias}}-\ws)
\end{align}
The second recursion corresponds to the dependence on the noise (termed as variance), wherein, the process is initiated at the solution, i.e. $\w_0^{\text{var}}=\ws$ and is driven by the noise $n_t$. The update for this process corresponds to:
\begin{align}
\label{eq:varRecc}
\w_t^{\text{var}}-\ws &= P_t (\w_{t-1}^{\text{var}}-\ws) + \eta_t n_t,\ \ \, \text{with }\ \ \w_0^{\text{var}}=\ws \\
&=\sum_{\tau=1}^t P_t\cdots P_{\tau+1} \cdot (\eta_\tau n_\tau) \nonumber.
\end{align}
We represent by $B_t$ the covariance of the $t^{\text{th}}$ iterate of the bias process, i.e., 
\begin{align*}
B_t &= \E{ \left(\w_t^{\text{bias}}-\ws\right) \left(\w_t^{\text{bias}}-\ws\right)^\top}\nonumber\\
&= \E{P_t B_{t-1} P_t^\top} = \E{P_t\cdots P_1 B_0 P_1\cdots P_t}
\end{align*}
The quantity that routinely shows up when bounding SGD's convergence behavior is the covariance of the variance error, i.e. $V_t:=\E{(\w_t^{\text{var}}-\ws)\otimes\w_t^{\text{var}}-\ws)}$.
This implies the following (simplified) expression for $V_t$:
\begin{align*}
V_t &= \E{ (\w_t^{\text{var}}-\ws) \otimes (\w_t^{\text{var}}-\ws) }\\
&= \E { \bigg(\sum_{\tau=1}^t P_t\cdots P_{\tau+1} \cdot (\eta_\tau n_\tau) \bigg) \otimes \bigg(\sum_{\tau'=1}^t P_t\cdots P_{\tau'+1} \cdot (\eta_\tau' n_\tau') \bigg) }\\
&= \sum_{\tau,\tau'} \E{P_T\cdots P_{\tau+1} (\eta_\tau n_\tau) \otimes (\eta_{\tau'} n_{\tau'}) P_{\tau'+1}\cdots P_T }\\
&= \sum_{\tau=1}^T \eta_\tau^2\E{P_T\cdots P_{\tau+1} n_\tau\otimes n_\tau P_{\tau+1}\cdots P_T}
\end{align*}
Firstly, note that this naturally implies that the sequence of covariances $V_\tau$, $\tau=1,\cdots,T$ initialized at (say), the solution, i.e., $\V_0 = 0$ naturally grows to its steady state covariance, i.e., 
\begin{align*}
V_1\preceq V_2 \preceq \cdots \preceq V_\infty.
\end{align*}
See lemma~3 of~\cite{JainKKNPS17} for more details. Furthermore, what naturally follows in relating $V_t$ to $V_{t-1}$ is:
\begin{align}
\label{eq:covarianceEvolution}
V_t \preceq \E{P_t V_{t-1} P_t^\top} + \eta_t^2 \sigma^2\Cov.
\end{align}
\end{proof}

\begin{lemma}
	[Lemma~5 of~\cite{JainKKNPS17}] Running SGD with a (constant) stepsize sequence $\eta<1/\infbound$ achieves the following steady-state covariance:
	\begin{align*}
	V_\infty \preceq \frac{\eta\sigma^2}{1-\eta\infbound}\eye.
	\end{align*}
\end{lemma}

\begin{lemma}
	\label{lem:CovarBound}
	Suppose $\eta=1/{2\infbound}$, and $V_0= \frac{\eta\sigma^2}{1-\eta\infbound}\eye=2\eta\sigma^2\eye$. 
	For any sequence of learning rates $\eta_t\leq\eta = 1/{2\infbound}\ \forall\ t\in\{1,\cdots,t\}$, then,
	\begin{align*}
	V_t \preceq 2\eta\sigma^2\eye\ \ \ \forall\ \ \ t.
	\end{align*}
\end{lemma}

\begin{proof}
	We will prove the lemma using an inductive argument. The base case, i.e. $t=0$ follows from the problem statement. Note also that for SGD, $V_0 = 0$ implying the statement naturally follows. If, say, $V_t$ satisfies the equation above, from equation~\ref{eq:covarianceEvolution}, we have the following covariance for $V_{t+1}$:
	\begin{align*}
	V_{t+1} &\preceq \E{P_t V_t P_t^\top} + \eta_t^2\sigma^2\Cov\\
	&= \E{ (\eye-\eta_t \x_t\x_t^\top) V_t (\eye - \eta_t \x_t \x_t^\top) } + \eta_t^2\sigma^2 \Cov \\
	&\preceq 2\eta\sigma^2 \E{ (\eye-\eta_t \x_t\x_t^\top)  (\eye - \eta_t \x_t \x_t^\top) } + \eta_t^2\sigma^2 \Cov\\
	&\preceq 2\eta\sigma^2\eye - 4\eta_t\eta\sigma^2 \Cov + 2\eta_t^2\eta\sigma^2\infbound \Cov + \eta_t^2\sigma^2\Cov\\
	&\preceq 2\eta\sigma^2\eye - 2\eta_t \eta \sigma^2 \Cov + \eta_t^2 \sigma^2\Cov\\
	&= 2\eta\sigma^2\eye + \eta_t\cdot(\eta_t-2\eta) \sigma^2\Cov \\
	&\preceq 2\eta\sigma^2\eye,
	\end{align*}
	from which the lemma follows.
\end{proof}

\begin{lemma}\label{lem:varUB}
	{\bf (Reduction from Multiplicative noise oracle)} Let $V_t$ be the (expected) covariance of the variance error. Then, the recursion that connects $V_{t+1}$ to $V_t$ can be expressed as:
	\begin{align*}
	V_{t+1} \preceq (\eye-\eta_t \Cov) V_t (\eye-\eta_t\Cov) + 2\eta_t^2 \sigma^2\Cov
	\end{align*}
\end{lemma}
\begin{proof}
	From equation~\ref{eq:covarianceEvolution}, we already know that the evolution of the co-variance of the variance error follows:
	\begin{align*}
	V_{t+1} &\preceq \E{P_t V_t P_t^\top} + \eta_t^2 \sigma^2 \Cov\\
	&\preceq \E{(\eye-\eta_t\Cov)V_t(\eye-\eta_t\Cov)} + \eta_t^2 \E{\x_t \x_t^\top V_t \x_t \x_t^\top} + \eta_t^2 \sigma^2 \Cov\\
	&\preceq  (\eye-\eta_t\Cov)V_t(\eye-\eta_t\Cov) + \eta_t^2 \norm{V_t}_2 \infbound\Cov + \eta_t^2\sigma^2\Cov\\
	&= (\eye-\eta_t\Cov)V_t(\eye-\eta_t\Cov) + \eta_t^2 \cdot 2\eta\sigma^2\infbound\Cov + \eta_t^2\sigma^2\Cov\\
	&\preceq (\eye-\eta_t\Cov)V_t(\eye-\eta_t\Cov) + 2\eta_t^2\sigma^2\Cov.
	\end{align*}
	Where the steps follow from lemma~\ref{lem:CovarBound}, and owing from the fact that $\eta_t\leq \eta = 1/2\infbound\ \ \forall\ \ t$. 
\end{proof}

{\bf Note:} Basically, one could analyze an auxiliary process driven by noise with variance off by a factor of two and convert the analysis into one involving exact (deterministic) gradients.

\begin{lemma}\label{lem:biasDecayStronglyConvex}
	[{\bf Bias decay - strongly convex case}] Let the minimal eigenvalue of the Hessian $\mu = \lamminH >0 $. Consider the bias recursion as in equation~\ref{eq:biasRecc} with the stepsize set as $\eta = 1/(2\infbound)$. Then, 
	\begin{align*}
	\E{\twonorm{\w_t^\text{bias}-\ws}^2}\leq (1-1/(2\kappa)) \E{\twonorm{\w_{t-1}^\text{bias}-\ws}^2}
	\end{align*}
\end{lemma}
\begin{proof}
	The proof follows through straight forward computations:
	\begin{align*}
	\E{\twonorm{\w_t^\text{bias}-\ws}^2} &\leq \E{\twonorm{\w_{t-1}^\text{bias}-\ws}^2}-2\eta\E{\norm{\w_{t-1}^\text{bias}-\ws}^2_\Cov} + \eta^2 \infbound \E{\norm{\w_{t-1}^\text{bias}-\ws}^2_\Cov}\\
	&= \E{\twonorm{\w_{t-1}^\text{bias}-\ws}^2}-\eta\E{\norm{\w_{t-1}^\text{bias}-\ws}^2_\Cov}\\
	&\leq (1-\eta\mu) \E{\twonorm{\w_{t-1}^\text{bias}-\ws}^2},
	\end{align*}
	where, the first line follows from the fact that $\E{\twonorm{\x_t}^2\x_t\x_t^\top}\preceq \infbound\Cov$ and the result follows through the definition of $\kappa$.
\end{proof}

\begin{lemma}\label{lem:biasUB}
	[{\bf Reduction of the bias recursion with multiplicative noise to one resembling the variance recursion}] Consider the bias recursion that evolves as $$B_t = \E{(\w_t-\ws)(\w_t-\ws)^\top} = \E{ (\eye-\gamma_t \x_t\x_t^\top) B_{t-1} (\eye-\gamma_t \x_t\x_t^\top)} \ \ \ \text{with}\ \  B_0 = (\w_0-\ws)(\w_0-\ws)^\top.$$ Then, the following recursion holds $\forall \gamma_t \leq 1/\infbound$:
	\begin{align*}
	B_t \preceq (\eye-\gamma_t\Cov)B_{t-1}(\eye-\gamma_t\Cov) + \gamma_t^2 \infbound \norm{\w_0-\ws}^2 \Cov.
	\end{align*}
\end{lemma}

\begin{proof}
	The result follows owing to the following computations:
	\begin{align*}
	B_t &= \E{(\w_t-\ws)(\w_t-\ws)^\top}\nonumber\\
	&= \E{(\eye-\gamma_t\x_t\x_t^\top) B_{t-1} (\eye-\gamma_t\x_t\x_t^\top)}\nonumber \\
	&\preceq (\eye-\gamma_t\Cov)B_{t-1}(\eye-\gamma_t\Cov)+\gamma_t^2\E{( \x_t^\top B_{t-1}\x_t)\x_t\x_t^\top}\nonumber\\
	&\preceq (\eye-\gamma_t\Cov)B_{t-1}(\eye-\gamma_t\Cov)+\gamma_t^2 \E{\twonorm{B_{t-1} }}\infbound \Cov\\
	&\preceq (\eye-\gamma_t\Cov)B_{t-1}(\eye-\gamma_t\Cov)+\gamma_t^2 \E{\twonorm{\w_{t-1}-\ws}^2}\infbound \Cov\\
	&\preceq (\eye-\gamma_t\Cov)B_{t-1}(\eye-\gamma_t\Cov)+\gamma_t^2 \E{\twonorm{\w_{0}-\ws}^2}\infbound \Cov,
	\end{align*}
	with the last inequality holding true if the squared distance to the optimum doesn't grow as a part of the recursion. We prove that this indeed is the case below:
	\begin{align*}
	\E{\twonorm{\w_{t-1}-\ws}^2} &= \E{\twonorm{\w_{t-2}-\gamma_{t-1}\x_{t-1}\x_{t-1}^\top -\ws}^2}\nonumber\\
	&\leq\E{\twonorm{\w_{t-2}-\ws}^2}-2\gamma_{t-1}\E{\norm{\w_{t-2}-\ws}^2_\Cov}+\gamma_{t-1}^2\infbound \E{\norm{\w_{t-2}-\ws}^2_\Cov}\\
	&\leq \E{\twonorm{\w_{t-2}-\ws}^2}-\gamma_{t-1}\E{\norm{\w_{t-2}-\ws}^2_\Cov}\nonumber\\
	&\leq \E{\twonorm{\w_{t-2}-\ws}^2}.
	\end{align*}
	Recursively applying the above argument yields the desired result.
\end{proof}
{\bf Note:} This result implies that the bias error (in the smooth non-strongly convex case of the least squares regression with multiplicative noise) can be bounded by employing a similar lemma as that of the variance, where one can look at the quantity $\infbound\cdot\|\w_0-\ws\|^2_2$ as the analog of the variance $\sigma^2$ that drives the process.

\begin{lemma}\label{lem:lbreduction}
	[{\bf Lower bounds on the additive noise oracle imply ones for the multiplicative noise oracle}] Under the assumption that the covariance of noise $\Sigma=\sigma^2\Cov$, the following statement holds. Let $V_t$ be the (expected) covariance of the variance error. Then, the recursion that connects $V_{t+1}$ to $V_t$ can be expressed as:
	\begin{align*}
	V_{t+1} = \E{(\eye-\eta_t \x_t\x_t^\top) V_t (\eye-\eta_t\x_t\x_t^\top)} + \eta_t^2 \sigma^2\Cov
	\end{align*}
	Then,
	\begin{align*}
	V_{t+1} \succeq (\eye-\eta_t \Cov) V_t (\eye-\eta_t\Cov) + \eta_t^2 \sigma^2\Cov
	\end{align*}	
\end{lemma}
\begin{proof}
	Let us consider firstly, the setting of (bounded) additive noise. Here, we have:
	\begin{align*}
	\hat{\nabla f}(\w_t) = \Cov (\w_t-\ws) + \zeta_t, \text{ with } \E{\zeta_t|\w_t} = 0 \text{, and } \E{\zeta_t\zeta_t^\top|\w_t} = \sigma^2\Cov. 
	\end{align*}
	Then, updates leading upto time $t+1$ can be written as:
	\begin{align*}
	\w_{t+1}-\ws = \prod_{\tau=1}^{t+1} (\eye-\eta_\tau\Cov) (\w_0-\ws) + \sum_{\tau'=1}^{t+1} \eta_{\tau'}\prod_{\tau=\tau'+1}^{t+1} (\eye-\eta_\tau\Cov) \zeta_{\tau'}
	\end{align*}
	This implies the covariance of the variance error is:
	\begin{align*}
	\tilde{V}_{t+1} &= \E{\left(\sum_{\tau'=1}^{t+1}\eta_{\tau'} \prod_{\tau=\tau'+1}^{t+1} (\eye-\eta_\tau\Cov) \zeta_{\tau'}\right)\otimes \left(\sum_{\tau^{''}=1}^{t+1}\eta_{\tau^{''}} \prod_{\tau=\tau^{''}+1}^{t+1} (\eye-\eta_\tau\Cov) \zeta_{\tau^{''}}\right)}\nonumber\\
	&= \sum_{\tau'=1}^{t+1}\eta_{\tau'}^2\E{\prod_{\tau=\tau'+1}^{t+1} (\eye-\eta_\tau\Cov) \zeta_{\tau'}\otimes\zeta_{\tau'} \prod_{\tau=t+1}^{\tau'+1} (\eye-\eta_\tau\Cov)}\\
	&= (\eye-\eta_{t+1}\Cov)V_{t}(\eye-\eta_{t+1}\Cov) + \eta_{t+1}^2\sigma^2\Cov.
	\end{align*}
	Now, let us consider the statement of the lemma:
	\begin{align*}
	V_{t+1} &= \E{(\eye-\eta_{t+1} \x_{t+1}\x_{t+1}^\top) V_t (\eye-\eta_{t+1}\x_{t+1}\x_{t+1}^\top)} + \eta_{t+1}^2 \sigma^2\Cov\\
	&= (\eye-\eta_{t+1} \Cov) V_t (\eye-\eta_{t+1}\Cov) + \eta_{t+1}^2\E{(\x_{t+1}\x_{t+1}^\top-\Cov)V_{t}(\x_{t+1}\x_{t+1}^\top-\Cov)} + \eta_{t+1}^2 \sigma^2\Cov\\
	&\succeq (\eye-\eta_{t+1} \Cov) V_t (\eye-\eta_{t+1}\Cov) + \eta_t^2\sigma^2\Cov.
	\end{align*}
	Unrolling the above argument and straightforward induction, we see that $V_{t+1}\succeq \tilde{V}_{t+1}$, implying that the process driven by the multiplicative noise oracle can be lower bounded (in a PSD sense) by one that employs deterministic gradients with additive noise.
\end{proof}

\section{Proofs of results in Section~\ref{sec:poly}}
\begin{theorem}\label{thm:poly-additive}
	Consider the additive noise oracle setting, where, we have access to stochastic gradients satisfying:
	\begin{align*}
	\widehat{\nabla f}(\w) = \nabla f(\w) + \zeta = \Cov (\w-\ws) + \zeta,
	\end{align*}
	where, 
	\begin{align*}
	\E{\zeta|\w} = 0 \text{, and, } \E{\zeta\zeta^\top|\w} = {\sigma}^2\Cov
	\end{align*}
	
	The following lower bounds hold on the final iterate of a Stochastic Gradient procedure with access to the above stochastic gradients when using polynomially decaying stepsizes.
	
	\noindent\textbf{Strongly convex case}: Suppose $\mu > 0$. For any condition number $\kappa$, there exists a problem instance with initial suboptimality $f(\w_0) - f(\ws) \le \sigma^2 d$ such that, for any $T \ge \kappa^{\frac{4}{3}}$, and for all $a,b\geq 0$ and $0.5 \leq \alpha \leq 1$, and for the learning rate scheme $\eta_t = \frac{a}{b+t^{\alpha}}$, we have
	\begin{align*}
	\E{f(\w_T)} -f(\ws) \geq \exp\left(-\frac{T}{\kappa\log T}\right)\left(f(\w_0)-f(\ws)\right) + \frac{\sigma^2 d}{{64}} \cdot \frac{\kappa}{T} .
	\end{align*}
	\textbf{Smooth case}: For any fixed $T> 1$, there exists a problem instance such that, for all $a,b\geq 0$ and $0.5 \leq \alpha \leq 1$, and for the learning rate scheme $\eta_t = \frac{a}{b+t^{\alpha}}$, we have
	\begin{align*}
	\E{f(\w_T)} -f(\ws) \geq \left( L \cdot \norm{\w_{0} - \ws}^2 + \sigma^2 d \right) \cdot \frac{1}{\sqrt{T} \log T}.
	\end{align*}
\end{theorem}
\begin{proof}
	\textbf{Strongly convex case}:
	The problem instance is simple. Consider the case where the inputs are such that in every example $\x$, there is only one co-ordinate that is non-zero. Furthermore, let each co-ordinate be Gaussian with mean zero and variance for the first $d/2$ co-ordinates be $d\kappa/3$ whereas the rest be $1$. This implies $\Cov = \begin{bmatrix}
	d\kappa/3 &  &  & \\  & \ddots &  & \\  & & {1}{} &  \\  & & & \ddots \end{bmatrix}$, where the first $\frac{d}{2}$ diagonal entries are equal to $\kappa/3$ and the remaining $\frac{d}{2}$ diagonal entries are equal to ${1}{}$ and all the off diagonal entries are equal to zero. Furthermore, consider the noise to be additive (and independent of $\x$) with mean zero. Finally, let us denote by $\var{i}_t \defeq \E{\left(\w^{(i)}_t - \left(\ws\right)^{(i)}\right)^2}$ the variance in the $i^\textrm{th}$ direction at time step $t$. Let the initialization be such that $\var{i}_0 = 3\sigma^2/\kappa$ for $i = 1,2,...,d/2$ and $\var{i}_0 = \sigma^2$ for $i = d/2+1,...,d$. 
	This means that the variances for all directions with eigenvalue $\kappa$ remain equal as $t$ progresses and similarly for all directions with eigenvalue $1$. We have
	\begin{align*}
	\var{1}_T &\eqdef \E{\left(\w^{(1)}_T - \left(\ws\right)^{(1)}\right)^2} = \prod_{j = 1}^{T} \left(1-\eta_j \kappa/3\right)^2 \var{1}_0 + \kappa \sigma^2/3 \sum_{j=1}^{T} \eta_j^2 \prod_{i=j+1}^{T} \left(1-\eta_i \kappa/3 \right)^2 \mbox{ and } \\
	\var{d}_T &\eqdef \E{\left(\w^{(d)}_T - \left(\ws\right)^{(d)}\right)^2} = \prod_{j = 1}^{T} \left(1-{\eta_j}{}\right)^2 \var{d}_0 + {\sigma^2}{} \sum_{j=1}^{T} \eta_j^2 \prod_{i=j+1}^{T} \left(1-{\eta_i}{}\right)^2.
	\end{align*}

	We consider a recursion for $\var{i}_t$ with eigenvalue $\lambda_i$ ($\kappa$ or ${1}{}$). By the design of the algorithm, we know
	$$
	\var{i}_{t+1} = (1-\eta_t\lambda_i)^2 \var{i}_t + \lambda_i \sigma^2 \eta_t^2.
	$$
	Let $s(\eta,\lambda) = \frac{\lambda\sigma^2\eta^2}{1 - (1-\eta\lambda)^2}$ be the solution to the stationary point equation $x = (1-\eta\lambda)^2 + \lambda\sigma^2\eta^2$. Intuitively if we keep using the same learning rate $\eta$, then $\var{i}_t$ is going to converge to $s(\eta,\lambda_i)$. Also note that $s(\eta,\lambda) \approx \sigma^2\eta/2$ when $\eta\lambda \ll 1$. 
	
	We first prove the following claim showing that eventually the variance in direction $i$ is going to be at least $s(\eta_{T}, \lambda_i)$. 
	
	\begin{claim}\label{clm1}
		Suppose $s(\eta_{t}, \lambda_i) \le \var{i}_0$, then $\var{i}_t \ge s(\eta_t,\lambda_i)$. 
	\end{claim}
	
	\begin{proof}
		We can rewrite the recursion as
		$$
		\var{i}_{t+1} - s(\eta_t,\lambda_i) = (1-\eta_t\lambda_i)^2 (\var{i}_t - s(\eta_t,\lambda_i)).
		$$
		In this form, it is easy to see that the iteration is a contraction towards $s(\eta_t,\lambda_i)$. Further, $\var{i}_{t+1} - s(\eta_t,\lambda_i)$ and $\var{i}_{t} - s(\eta_t,\lambda_i)$ have the same sign. In particular, let $t_0$ be the first time such that $s(\eta_{t}, \lambda_i) \le \var{i}_0$ (note that $\eta_t$ is monotone and so is $s(\eta_{t}, \lambda_i)$), it is easy to see that $\var{i}_t \ge \var{i}_0$ when $t \le t_0$. Therefore we know $\var{i}_{t_0} \ge s(\eta_{t_0},\lambda_i)$, by the recursion this implies $\var{i}_{t_0+1} \ge s(\eta_{t_0},\lambda_i) \ge s(\eta_{t_0+1},\lambda_i)$. The claim then follows from a simple induction.
	\end{proof}
	
	If $s(\eta_T, \lambda_i) \ge \var{i}_0$ for $i = 1$ or $i = d$ then the error is at least $\sigma^2d/2 \ge \kappa\sigma^2d/T$ and we are done. Therefore we must have $s(\eta_T, \kappa) \le \var{1}_0 = 3\sigma^2/\kappa$, and by Claim~\ref{clm1} we know $\var{1}_T \ge s(\eta_T, \kappa) \ge \sigma^2\eta_T/2$. The function value is at least 
	$$
	\E{f(\w_T)}-f(\ws) \ge \frac{d}{2}\cdot \kappa \cdot \var{1}_T \ge \frac{d \kappa \sigma^2\eta_T}{12}.
	$$
	To make sure $\E{f(\w_T)}-f(\ws) \le \frac{d\kappa\sigma^2}{64T}$ we must have $\eta_T \le \frac{1}{6T}$. Next we will show that when this happens, $\var{d}_T$ must be large so the function value is still large.
	
	We will consider two cases, in the first case, $b \ge T^{\alpha}$. Since $\frac{1}{16T} \ge \eta_T = \frac{a}{b+T^\alpha} \ge \frac{a}{2b}$, we have $\frac{a}{b} \le \frac{1}{8T}$. Therefore $\var{d}_T \ge (1-\frac{a}{b})^{2T} \var{d}_0 \ge \sigma^2/2$, so the function value is at least $\E{f(\w_t)}\ge \frac{d}{2} \cdot \var{d}_T \ge \frac{d\sigma^2}{4 } \ge \frac{\kappa d\sigma^2}{T}$, and we are done.
	
	In the second case, $b < T^\alpha$. Since $\frac{1}{16T} \ge \eta_T = \frac{a}{b+T^\alpha} \ge \frac{a}{2T^{\alpha}}$, we have $a \le \frac{1}{8} T^{\alpha-1}$. The sum of learning rates satisfy
	
	$$\sum_{i=1}^T \eta_i \le \sum_{i=1}^T\frac{a}{i^\alpha} \le \sum_{i=1}^T \frac{1}{8} i^{-1} \leq 0.125 \log T.$$

	Here the second inequality uses the fact that $T^{\alpha - 1}i^{-\alpha} \le i^{-1}$ when $i \le T$. Similarly, we also know $\sum_{i=1}^T \eta_i^2 \le \sum_{i=1}^T (0.125)^2 i^{-2} \le \pi^2 /384$. Using the approximation $(1-u)^2 \ge \exp(-2u-4u^2)$ for $u < 1/4$, we get  $\var{d}_T \ge \exp(-2\sum_{i=1}^T {\eta_i}{}-4\sum_{i=1}^T{\eta_i^2}{}) \var{d}_0 \ge \sigma^2/5{T^{\frac{1}{4}}}$, so the function value is at least $\E{f(\w_t)}\ge \frac{d}{2}\cdot \var{d}_T \ge \frac{d\sigma^2}{10{T^{\frac{1}{4}}}} \ge \frac{\kappa d\sigma^2}{32T}$. This concludes the second case and proves the strongly convex part of the theorem.
	
	\textbf{Smooth case}: The proof of this part is quite similar to that of the strongly convex case above but with a subtle change in the initialization. In order to make this clear, we will do the proof from scratch with out borrowing anything from the previous argument.
	Let $\Cov = \begin{bmatrix}
	1 &  &  & \\  & \ddots &  & \\  & & \frac{d}{\kappa} &  \\  & & & \ddots \end{bmatrix}$, where the first $\frac{d}{2}$ diagonal entries are equal to $1$ and the remaining $\frac{d}{2}$ diagonal entries are equal to $\frac{d}{\kappa}$ and all the off diagonal entries are equal to zero. We will use $\kappa = \frac{1}{\sqrt{T}}$. Let us denote by $\var{i}_t \defeq \E{\left(\w^{(i)}_t - \left(\ws\right)^{(i)}\right)^2}$ the variance in the $i^\textrm{th}$ direction at time step $t$. Let the initialization be such that $\var{i}_0 = \sigma^2/\kappa$ for $i = 1,2,...,d/2$ and $\var{i}_0 = \sigma^2$ for $i = d/2+1,...,d$. 
	This means that the variances for all directions with eigenvalue $\kappa$ remain equal as $t$ progresses and similarly for all directions with eigenvalue $1$. We have
	\begin{align*}
	\var{1}_T &\eqdef \E{\left(\w^{(1)}_T - \left(\ws\right)^{(1)}\right)^2} = \prod_{j = 1}^{T} \left(1-\eta_j \kappa/3\right)^2 \var{1}_0 + \kappa \sigma^2/3 \sum_{j=1}^{T} \eta_j^2 \prod_{i=j+1}^{T} \left(1-\eta_i \kappa/3 \right)^2 \mbox{ and } \\
	\var{d}_T &\eqdef \E{\left(\w^{(d)}_T - \left(\ws\right)^{(d)}\right)^2} = \prod_{j = 1}^{T} \left(1-{\eta_j}{}\right)^2 \var{d}_0 + {\sigma^2}{} \sum_{j=1}^{T} \eta_j^2 \prod_{i=j+1}^{T} \left(1-{\eta_i}{}\right)^2.
	\end{align*}

	We consider a recursion for $\var{i}_t$ with eigenvalue $\lambda_i$ (1 or $\frac{1}{\kappa}$). By the design of the algorithm, we know
	$$
	\var{i}_{t+1} = (1-\eta_t\lambda_i)^2 \var{i}_t + \lambda_i \sigma^2 \eta_t^2.
	$$
	Let $s(\eta,\lambda) = \frac{\lambda\sigma^2\eta^2}{1 - (1-\eta\lambda)^2}$ be the solution to the stationary point equation $x = (1-\eta\lambda)^2 + \lambda\sigma^2\eta^2$. Intuitively if we keep using the same learning rate $\eta$, then $\var{i}_t$ is going to converge to $s(\eta,\lambda_i)$. Also note that $s(\eta,\lambda) \approx \sigma^2\eta/2$ when $\eta\lambda \ll 1$. 
	
	%
	%
	%
	If $s(\eta_T, \lambda_i) \ge \var{i}_0$ for $i = 1$ or $i = d$ then the error is at least $\sigma^2d/2\kappa \ge \kappa\sigma^2d/T$ and we are done. Therefore we must have $s(\eta_T, \kappa) \le \var{1}_0 = 3\sigma^2/\kappa$, and by Claim~\ref{clm1} we know $\var{1}_T \ge s(\eta_T, \kappa) \ge \sigma^2\eta_T/2$. The function value is at least 
	$$
	\E{f(\w_T)}-f(\ws) \ge \frac{d}{2}\cdot \var{1}_T \ge \frac{d\sigma^2\eta_T}{4}.
	$$
	To make sure $\E{f(\w_T)}-f(\ws) \le \frac{d\kappa\sigma^2}{64T\log T}$ we must have $\eta_T \le \frac{\kappa}{16T\log T}$. Next we will show that when this happens, $\var{d}_T$ must be large so the function value is still large.
	
	We will consider two cases, in the first case, $b \ge T^{\alpha}$. Since $\frac{\kappa}{16T\log T} \ge \eta_T = \frac{a}{b+T^\alpha} \ge \frac{a}{2b}$, we have $\frac{a}{b} \le \frac{\kappa}{8T\log T}$. Therefore $\var{d}_T \ge (1-\frac{a}{b})^{2T} \var{d}_0 \ge \sigma^2/2$, so the function value is at least $\E{f(\w_t)}-f(\ws)\ge \frac{d}{2} \cdot \frac{1}{\kappa} \cdot \var{d}_T \ge \frac{d\sigma^2}{4 \kappa} \ge \frac{\kappa d\sigma^2}{T}$, and we are done.
	
	In the second case, $b < T^\alpha$. Since $\frac{\kappa}{16T\log T} \ge \eta_T = \frac{a}{b+T^\alpha} \ge \frac{a}{2T^{\alpha}}$, we have $a \le \frac{1}{8\log T} \kappa T^{\alpha-1}$. The sum of learning rates satisfy
	
	$$\sum_{i=1}^T \eta_i \le \sum_{i=1}^T\frac{a}{i^\alpha} \le \sum_{i=1}^T \frac{1}{8\log T} \kappa i^{-1} \leq 0.125 \kappa.$$

	Here the second inequality uses the fact that $T^{\alpha - 1}i^{-\alpha} \le i^{-1}$. Similarly, we also know $$\sum_{i=1}^T \eta_i^2 \le \sum_{i=1}^T (0.125 \kappa/\log T)^2 i^{-2} \le \pi^2 \kappa^2 /384.$$ Using the approximation $(1-u)^2 \ge \exp(-2u-4u^2)$ for $u < 1/4$, we get  $\var{d}_T \ge \exp(-2\sum_{i=1}^T \frac{\eta_i}{\kappa}-4\sum_{i=1}^T\frac{\eta_i^2}{\kappa^2}) \var{d}_0 \ge \sigma^2/5$, so the function value is at least $\E{f(\w_t)}\ge \frac{d}{2}\cdot \frac{1}{\kappa} \cdot \var{d}_T \ge \frac{d\sigma^2}{10\kappa} \ge \frac{d\sigma^2}{10\sqrt{T}}$. This concludes the second case and proves the strongly convex part of the theorem. Since $\norm{\Cov} \cdot \norm{\w_{0} - \ws}^2 = d \sigma^2$, we have
	\begin{align*}
	\E{f(\w_{T})} - f(\ws) \geq \sigma^2 d \cdot \min\left(\frac{\kappa}{T \log T}, \frac{1}{10 \sqrt{T}}\right) \geq \left( L \cdot \norm{\w_{0} - \ws}^2 + \sigma^2 d \right) \cdot \frac{1}{\sqrt{T} \log T}.
	\end{align*}
	This proves the theorem.
\end{proof}

\begin{proof}[Proof of Theorem~\ref{thm:poly}]
	The proof of theorem~\ref{thm:poly} follows straightforwardly when combining the result of lemma~\ref{lem:lbreduction} and theorem~\ref{thm:poly-additive}.	
\end{proof}

\section{Proofs of results in Section~\ref{sec:const-cut}}

\begin{theorem}\label{thm:const-cut-additive}
	Consider the additive noise oracle setting, where, we have access to stochastic gradients satisfying:
	\begin{align*}
	\widehat{\nabla f}(\w) = \nabla f(\w) + \zeta = \Cov (\w-\ws) + \zeta,
	\end{align*}
	where, 
	\begin{align*}
	\E{\zeta|\w} = 0 \text{, and, } \E{\zeta\zeta^\top|\w} \preceq \hat{\sigma}^2\Cov
	\end{align*}
	Running Algorithm~\ref{algo:StepDecay} with an initial stepsize of $\eta_1=1/\infbound$, starting from the solution, i.e. $\w_0=\ws$ allows the algorithm to obtain the following dependence on the variance error:
	\begin{align*}
	\E{f(\w_T^{\text{var}})} - f(\ws) \leq 2 \frac{d\hat{\sigma}^2\log T}{T}
	\end{align*}
\end{theorem}
\begin{proof}
	The learning rate scheme is as follows. Divide the total time horizon $T$ into $\log T$ phases, each of length $\frac{T}{\log T}$. In the $\ell^\textrm{th}$ phase, the learning rate is set to be $\frac{1}{2^\ell \infbound}$. The variance in the $k^\textrm{th}$ coordinate can be bounded as
	\begin{align}
	\var{k}_T &\leq \prod_{j = 1}^{T} \left(1-\eta_j \lambda^{(k)} \right)^2 \var{k}_{0} + \lambda^{(k)} \hat\sigma^2 \sum_{j=1}^{T} \eta_j^2 \prod_{i=j+1}^{T} \left(1-\eta_i \lambda^{(k)} \right)^2 \nonumber\\
	&\leq \exp\left(-2\sum_{j = 1}^{T} \eta_j \lambda^{(k)} \right) \var{k}_0 \nonumber\\ &\qquad + \lambda^{(k)} \hat\sigma^2 \sum_{\ell=1}^{\log T} \frac{1}{2^{2\ell}(\infbound)^2} \sum_{j=1}^{T/\log T} \left(1-\frac{\lambda^{(k)}}{2^{\ell}(\infbound)} \right)^{2j} \cdot \prod_{u = \ell+1}^{\log T}\left(1-\frac{\lambda^{(k)}}{2^u \infbound}\right)^{T/\log T} \nonumber\\
	&\leq \exp\left(-\frac{2 \lambda^{(k)}}{\infbound} \cdot \frac{T}{\log T }\right) \var{k}_0 + \lambda^{(k)} \hat\sigma^2 \sum_{\ell=1}^{\log T} \frac{1}{2^{2\ell}(\infbound)^2} \cdot \frac{2^\ell \infbound}{\lambda^{(k)}} \cdot \prod_{u = \ell+1}^{\log T}\exp\left(-\frac{\lambda^{(k)}T}{2^u \infbound\log T}\right) \nonumber\\
	&\leq \exp\left(-\frac{2 \lambda^{(k)}}{\infbound} \cdot \frac{T}{\log T }\right) \var{k}_0 + \sum_{\ell=1}^{\log T} \frac{\hat\sigma^2}{2^\ell \infbound} \prod_{u = \ell+1}^{\log T}\exp\left(-\frac{\lambda^{(k)}T}{2^u \infbound\log T}\right).\label{eq:var}
	\end{align}
Let $\ell^*\eqdef \max\left(0, \floor{\log\left(\frac{\lambda^{(k)}}{\infbound} \cdot \frac{T}{\log T}\right)}\right)$. We now split the summation in the second term in~\eqref{eq:var} into two parts and bound each of them below.
\begin{align}
	&\sum_{\ell=1}^{\ell^*} \frac{\hat\sigma^2}{2^\ell \infbound} \prod_{u = \ell+1}^{\log T}\exp\left(-\frac{\lambda^{(k)}T}{2^u \infbound\log T}\right) \leq \sum_{\ell=1}^{\ell^*} \frac{\hat\sigma^2}{2^\ell \infbound} \prod_{u = \ell+1}^{\ell^*}\exp\left(-\frac{\lambda^{(k)}T}{2^u \infbound\log T}\right) \nonumber \\
	&\leq \sum_{\ell=1}^{\ell^*} \frac{\hat\sigma^2}{2^\ell \infbound} \prod_{u = \ell+1}^{\ell^*}\exp\left(-2^{\ell^*-u}\right)
	\leq \sum_{\ell=1}^{\ell^*} \frac{\hat\sigma^2}{2^\ell \infbound} \exp\left(-2^{\ell^*-\ell}\right) \nonumber \\
	&\leq \frac{\hat\sigma^2}{2^{\ell^*} \infbound} \sum_{\ell=1}^{\ell^*} 2^{\ell^* - \ell} \exp\left(-2^{\ell^*-\ell}\right)
	\leq \frac{\hat\sigma^2}{2^{\ell^*} \infbound} \leq \frac{\hat\sigma^2}{\lambda^{(k)}} \cdot \frac{\log T}{T}. \label{eq:var-1}
\end{align}
For the second part, we have
\begin{align}
	&\sum_{\ell=\ell^* + 1}^{\log T} \frac{\hat\sigma^2}{2^\ell \infbound} \prod_{u = \ell+1}^{\log T}\exp\left(-\frac{\lambda^{(k)}T}{2^u \infbound\log T}\right) \leq \sum_{\ell=\ell^* + 1}^{\log T} \frac{\hat\sigma^2}{2^\ell \infbound} \leq \sum_{\ell=\ell^* + 1}^{\log T} \frac{\hat\sigma^2}{2^{\ell^*}\infbound} \leq \frac{\hat\sigma^2}{\lambda^{(k)}} \cdot \frac{\log T}{T}. \label{eq:var-2}
\end{align}
Plugging~\eqref{eq:var-1} and~\eqref{eq:var-2} into~\eqref{eq:var}, we obtain
\begin{align*}
	\var{k}_T &\leq \exp\left(-\frac{2 \lambda^{(k)}}{\infbound} \cdot \frac{T}{\log T }\right) \var{k}_0 + \frac{2 \hat\sigma^2}{\lambda^{(k)}} \cdot \frac{\log T}{T}.
\end{align*}
The function suboptimality can now be bounded as
\begin{align}
	\E{f(\w_{T}^{\text{var}})} - f(\ws) &= \sum_{k = 1}^{d} \lambda^{(k)} \cdot \var{k}_T \nonumber \\
	&\leq \sum_{k = 1}^{d} \lambda^{(k)} \left(\exp\left(-\frac{2 \lambda^{(k)}}{\infbound} \cdot \frac{T}{\log T }\right) \var{k}_0 + \frac{2 \hat\sigma^2}{\lambda^{(k)}} \cdot \frac{\log T}{T}\right).\nonumber 
\end{align}
\begin{align*}
\E{f(\w_{T}^{\text{var}})} - f(\ws) &\leq \sum_{k = 1}^{d} \left(\frac{L \log T}{T} \var{k}_0 + {2 \hat\sigma^2} \cdot \frac{\log T}{T}\right)
= 2\left(\hat\sigma^2 d\right)\frac{\log T}{T}.
\end{align*}

\end{proof}

\begin{proof}[Proof of Theorem~\ref{thm:const-cut}]
	{\bf Smooth case:} The result follows by instantiating $\hat\sigma^2$ in theorem~\ref{thm:const-cut-additive} with $2 \sigma^2$ (lemma~\ref{lem:varUB}) and $\infbound\twonorm{\w_0-\ws}^2$ (lemma~\ref{lem:biasUB}) and using the lemma~\ref{lem:biasVarTradeoffNew} to obtain the result.
	
	{\bf Strongly convex case:} As with the smooth case, the result relies on instantiating theorem~\ref{thm:const-cut-additive} with $2\sigma^2$ (lemma~\ref{lem:varUB}) and using lemma~\ref{lem:biasDecayStronglyConvex} and then appealing to lemma~\ref{lem:biasVarTradeoffNew}.
\end{proof}

\begin{proposition}\label{thm:compl-additive}
	Consider the additive noise oracle setting, where, we have access to stochastic gradients satisfying:
	\begin{align*}
	\widehat{\nabla f}(\w) = \nabla f(\w) + \zeta = \Cov (\w-\ws) + \zeta,
	\end{align*}
	where, 
	\begin{align*}
	\E{\zeta|\w} = 0 \text{, and, } \E{\zeta\zeta^\top|\w} \leq {\sigma}^2\Cov
	\end{align*}
	There exists a stepsize scheme with which, by starting at the solution (i.e. $\w_0=\ws$) the algorithm obtains the following dependence on the variance error, under the assumption that $\mu>0$ and $\kappa\geq 2$.
	\begin{align*}
	\E{f(\w_T^{\text{var}})} - f(\ws) \leq {50\log_2 \kappa}{} \cdot \frac{\sigma^2 d}{T}.
	\end{align*}
\end{proposition}
\begin{proof}
	The learning rate scheme is as follows. 
	
	We first break $T$ into three equal sized parts. Let $A = T/3$ and $B = 2T/3$. In the first $T/3$ steps, we use a constant learning rate of $1/\infbound$. Note that at the end of this phase, (since $T>\kappa$) the dependence on the initial error decays geometrically. In the second $T/3$ steps, we use a polynomial decay learning rate $\eta_{A+t} = \frac{1}{\mu(\kappa + t/2)}$. In the third $T/3$ steps, we break the steps into $\log_2(\kappa)$ equal sized phases. In the $\ell^{\textrm{th}}$ phase, the learning rate to be used is $\frac{5 \log_2 \kappa}{2^\ell \cdot \mu \cdot T}$. Note that the learning rate in the first phase depends on strong convexity and that in the last phase depends on smoothness (since the last phase has $\ell = \log \kappa$). 
	
	Recall the variance in the $k^{\textrm{th}}$ coordinate can be upper bounded by
	\begin{align*}
	\var{k}_T &\eqdef \E{\left(\w^{(k)}_T - \left(\ws\right)^{(1)}\right)^2} \leq \prod_{j = 1}^{T} \left(1-\eta_j \lambda^{(k)} \right)^2 \var{1}_0 + \lambda^{(k)} \sigma^2 \sum_{j=1}^{T} \eta_j^2 \prod_{i=j+1}^{T} \left(1-\eta_i \lambda^{(k)} \right)^2 \\
	&\qquad \leq \exp\left(-2\sum_{j = 1}^{T} \eta_j \lambda^{(k)} \right) \var{1}_0 + \lambda^{(k)} \sigma^2 \sum_{j=1}^{T} \eta_j^2 \exp\left(-2\sum_{i=j+1}^{T} \eta_i \lambda^{(k)} \right).
	\end{align*}
	We will show that for every $k$, we have
	\begin{align}
	\var{k}_T \le \frac{\var{k}_0}{T^3} + \frac{50\log_2\kappa}{\lambda^{(k)}T}\cdot \sigma^2., \label{eq:upperboundcoordinate}
	\end{align}
	which directly implies the theorem.
	
	We will consider the first $T/3$ steps. The guarantee that we will prove for these iterations is: for any $t \le A$, $\var{k}_t \le (1-\lambda^{(k)}/\infbound)^{2t} \var{k}_0 + \frac{\sigma^2}{\infbound}$.
	
	This can be proved easily by induction. Clearly this is true when $t = 0$. Suppose it is true for $t-1$, let's consider step $t$. By recursion of $\var{k}_t$ we know
	\begin{align*}
	\var{k}_t & = (1-\lambda^{(k)}/\infbound)^2 \var{k}_{t-1} + \lambda^{(k)} \sigma^2/(\infbound)^2\\
	&\le (1-\lambda^{(k)}/\infbound)^{2t} \var{k}_0 + \frac{\sigma^2}{\infbound}\left( (1-\lambda^{(k)}/\infbound)^2 + \lambda^{(k)}/\infbound\right)\\
	& \le (1-\lambda^{(k)}/\infbound)^{2t} \var{k}_0 + \frac{\sigma^2}{\infbound}.
	\end{align*}
	Here the second step uses induction hypothesis and the third step uses the fact that $(1-x)^2+x \le 1$ when $x \in [0,1]$. In particular, since $(1-\lambda^{(k)}/\infbound)^{2T/3} \le (1-1/\kappa)^{2T/3} \le (1-1/\kappa)^{3\kappa \log T} = 1/T^3$, we know at the end of the first phase,
	$\var{k}_A \le \var{k}_0/T^3 + \frac{\sigma^2}{\infbound}$. 
	
	In the second $T/3$ steps, the guarantee would be: for any $t \le T/3$, $\var{k}_{A+t} \le \var{k}_0/T^3 + 2\eta_{A+t}\sigma^2$. 
	
	We will again prove this by induction. The base case ($t = 0$) follows immediately from the guarantee for the first part. Suppose this is true for $A+t -1$, let us consider $A+t$, again by recursion we know
	\begin{align*}
	\var{k}_{A+t} & = (1-\lambda^{(k)}\eta_{A+t-1})^2 \var{k}_{A+t-1} + \lambda^{(k)} \sigma^2\eta_{A+t-1}^2\\
	&\le \var{k}_0/T^3 + 2\eta_{A+t-1}\sigma^2\left( (1-\lambda^{(k)}\eta_{A+t-1})^2 + \frac{1}{2}\lambda^{(k)}\eta_{A+t-1}\right)\\
	& \le \var{k}_0/T^3 + 2\eta_{A+t-1}\sigma^2(1-\frac{1}{2}\mu\eta_{A+t-1}) \le \var{k}_0/T^3 + 2\eta_{A+t}\sigma^2.
	\end{align*}
	Here the last line uses the fact that $2\eta_{A+t-1}(1-\frac{1}{2}\mu\eta_{A+t-1}) \le 2\eta_{A+t}\sigma^2$, which is easy to verify by our choice of $\eta$. Therefore, at the end of the second part, we have $\var{k}_B \le \var{k}_0/T^3 + \frac{2\sigma^2}{\mu(\kappa+T/6)}$. 
	
	Finally we will analyze the third part. Let $\hat{T} = T/3\log_2\kappa$, we will consider the variance $\var{k}_{B+\ell \hat{T}}$ at the end of each phase. We will make the following claim by induction:
	
	\begin{claim}\label{clm2}
		Suppose $2^\ell\cdot \mu \le \lambda^{(k)}$, then 
		$$
		\var{k}_{B+\ell\hat{T}} \le \var{k}_B\exp(-3\ell) + 2\hat{T}\eta_{\ell}^2 \lambda^{(k)}\sigma^2.
		$$
	\end{claim}	
	
	\begin{proof} We will prove this by induction. When $\ell = 0$, clearly we have $\var{k}_B \le \var{k}_B$ so the claim is true. Suppose the claim is true for $\ell-1$, we will consider what happens after the algorithm uses $\eta_\ell$ for $\hat{T}$ steps. By the recursion of the variance we have
		$$
		\var{k}_{\ell\hat{T}} \le \var{k}_{(\ell-1)\hat{T}} \cdot \exp(-2\eta_\ell\cdot \lambda^{(k)} \hat{T}) + \hat{T}\eta_{\ell}^2 \lambda^{(k)}\sigma^2.
		$$
		Since $2^\ell\cdot \mu \le \lambda^{(k)}$, we know $\exp(-2\eta_\ell\cdot \lambda^{(k)} \hat{T}) \le \exp(-3)$. Therefore by induction hypothesis we have
		$$
		\var{k}_{B+\ell\hat{T}} \le \var{k}_B\exp(-3\ell) + \exp(-3)\cdot 2\hat{T}\eta_{\ell-1}^2 \lambda^{(k)} + \hat{T}\eta_{\ell}^2 \lambda^{(k)} \le \var{k}_B\exp(-3\ell) + 2\hat{T}\eta_{\ell}^2 \lambda^{(k)}.
		$$
		This finishes the induction.
	\end{proof}
	
	By Claim~\ref{clm2}, Let $\ell^*$ denote the number satisfying $2^{\ell^*}\cdot \mu \leq \lambda^{(k)} < 2^{\ell^*+1}\cdot \mu $, by this choice we know $\mu/\lambda^{(k)} \ge \frac{1}{2}\exp(-3\ell^\star)$ we have
	
	\begin{align*}
	\var{k}_{T} \le \var{k}_{B+\ell^*\hat{T}} &\le \var{k}_B\exp(-3\ell^*) + 2\hat{T}\eta_{\ell^*}^2 \lambda^{(k)}\sigma^2 \\
	&\le \frac{\var{k}_0}{T^3} + \frac{24\sigma^2}{\lambda^{(k)} T} + \frac{50\log_2\kappa}{3\lambda^{(k)}T}\cdot \sigma^2. \\
	& \le \frac{\var{k}_0}{T^3} + \frac{50\log_2\kappa}{\lambda^{(k)}T}\cdot \sigma^2.
	\end{align*}
	
	
	Therefore, the function value is bounded by $\E{f(\w_T^{\text{var}})}-f(\ws) = \sum_{i=1}^d \lambda^{(k)}\var{k}_T \le \frac{50 \log_2 \kappa}{T} \cdot \sigma^2 d$.
\end{proof}

\begin{proof}[Proof of proposition~\ref{thm:compl}]
	The proof of the proposition works similar to the proof of the strongly convex case of theorem~\ref{thm:const-cut}, wherein, we combine the result of proposition~\ref{thm:compl-additive} with lemma~\ref{lem:biasDecayStronglyConvex} and lemma~\ref{lem:biasVarTradeoffNew} to obtain the result.
\end{proof}

\section{Proofs of results in Section~\ref{sec:inf-horizon}}\label{app:inf-horizon}
All of our counter-examples in this section are going to be the same simple function. Let the inputs $x$ be such that only a single co-ordinate be active on each example. We refer to this case as the ``discrete'' case. Furthermore, let each co-ordinate be a Gaussian with mean $0$ and variance for the first $d/2$ directions being $d\kappa/3$ and the final $d/2$ directions being $1$. Furthermore, consider the noise to be additive (and independent of $\x$) with mean zero. This indicates that $\infbound=\kappa$ for this problem.

Intuitively, we will show that in order to have a small error in the first eigendirection (with eigenvalue $\kappa$), one need to set a small learning rate $\eta_t$ which would be too small to achieve a small error in the second eigendirection (with eigenvalue $1$). As a useful tool, we will decompose the variance in the two directions corresponding to $\kappa$ eigenvalue and $1$ eigenvalue respectively as follows:
\begin{align}
\var{1}_T &\eqdef \E{\left(\w^{(1)}_T - \left(\ws\right)^{(1)}\right)^2} = \prod_{j = 1}^{T} \left(1-\eta_j \kappa\right)^2 \var{1}_0 + \kappa \sigma^2 \sum_{j=1}^{T} \eta_j^2 \prod_{i=j+1}^{T} \left(1-\eta_i \kappa \right)^2 \nonumber \\
&\qquad \geq \exp\left(-2\sum_{j = 1}^{T} \eta_j \kappa \right) \var{1}_0 + \kappa \sigma^2 \sum_{j=1}^{T} \eta_j^2 \exp\left(-2\sum_{i=j+1}^{T} \eta_i \kappa \right) \label{eqn:exp-1}\mbox{ and}\\
\var{2}_T &\eqdef \E{\left(\w^{(2)}_T - \left(\ws\right)^{(2)}\right)^2} = \prod_{j = 1}^{T} \left(1-{\eta_j}{}\right)^2 \var{2}_0 + {\sigma^2}{} \sum_{j=1}^{T} \eta_j^2 \prod_{i=j+1}^{T} \left(1-{\eta_i}{}\right)^2 \nonumber\\
&\qquad \geq \exp\left(-2\sum_{j = 1}^{T} \eta_j \right) \var{2}_0 + {\sigma^2}{} \sum_{j=1}^{T} \eta_j^2 \exp\left(-2\sum_{i=j+1}^{T} \eta_i \right) .\label{eqn:exp-2}
\end{align}
\rong{Since this is used by every theorem I think it makes sense to put it in front.}
\rong{I also added a constant factor 2 in the exponentials because $(1-\epsilon)^2 \ge \exp(-2\epsilon)$ and I think it is not true without the factor 2.}

\begin{theorem}\label{thm:limsup-additive}
	Consider the additive noise oracle setting, where, we have access to stochastic gradients satisfying:
	\begin{align*}
	\widehat{\nabla f}(\w) = \nabla f(\w) + \zeta = \Cov (\w-\ws) + \zeta,
	\end{align*}
	where, 
	\begin{align*}
	\E{\zeta|\w} = 0 \text{, and, } \E{\zeta\zeta^\top|\w} = {\sigma}^2\Cov
	\end{align*}
	There exists a universal constant $C > 0$, and a problem instance, such that for SGD algorithm with any $\eta_t \le 1/2\kappa$ for all $t$\footnote{Learning rate more than $2/\kappa$ will make the algorithm diverge.}, we have
	\begin{align*}
	\limsup_{T\rightarrow \infty} \frac{\E{f(\w_T)}-f(\ws)}{\left(\sigma^2 d/T\right)} \geq C \frac{\kappa }{\log (\kappa+1)}
	\, .
	\end{align*}
\end{theorem}
\begin{proof}
	Fix $\tau = \kappa/C\log (\kappa+1)$ where $C$ is a universal constant that we choose later. We need to exhibit 
	that the $\limsup$ is larger than $\tau$. 
	For simplicity we will also round $\kappa$ up to the nearest integer. 
	
	Let $T$ be a given number. Our goal is to exhibit a $\nt > T$ such that $\frac{f(\w_{\nt})-f(\ws)}{\left(\sigma^2/\nt\right)} \geq \tau$. Given the step size sequence $\eta_t$, consider the sequence of numbers $T_0 =T, T_1,\cdots,T_\kappa$ such that $T_i$ is the first number that
	$$
	\frac{1}{\kappa} \le \sum_{t = T_{i-1}+1}^{T_i} \eta_t \le \frac{3}{\kappa}.
	$$
	
	Note that such a number always exists because all the step sizes are at most $2/\kappa$. We will also let $\dt_i$ be $T_i - T_{i-1}$.
	Firstly, from \eqref{eqn:exp-1} and~\eqref{eqn:exp-2}, we see that $\sum_t \eta_t = \infty$. Otherwise, the bias will never decay to zero.	If $f(\w_{T_{i-1}+ \dt_i})-f(\ws) > \frac{\tau\sigma^2 d}{T_{i-1}+\dt_i}$ for some $i = 1,\cdots,\kappa$, we are done. If not, we obtain the following relations:
	\begin{align*}
	&\quad \frac{\sigma^2}{\dt_1} \leq \sigma^2 \sum_{t = 1}^{\dt_1} \eta_{T_0+t}^2 \leq \frac{\exp(3)}{\kappa}\cdot \E{\left(\w_{T_0+\dt_1}^{(1)} - \left(\ws\right)^{(1)}\right)^2} \\&\leq \exp(3)(f(\w_{T_0+\dt_1})-f(\ws)) \leq \frac{\exp(3)\tau \sigma^2}{T_0+\dt_1} \\
	&\Rightarrow T_0 \leq \left(\exp(3)\tau-1\right) \dt_1.
	\end{align*}
	Here the second inequality is based on \eqref{eqn:exp-1}. We will use $C_1$ to denote $\exp(3)$. Similarly, we have
	\begin{align*}
	&\quad \frac{\sigma^2}{\dt_2} \leq \sigma^2 \sum_{t = 1}^{\dt_2} \eta_{T_1+t}^2 \leq \frac{C_1}{\kappa} \E{\left(\w_{T_1+\dt_2}^{(1)} - \left(\ws\right)^{(1)}\right)^2} \leq C_1 (f(\w_{T_1+\dt_2})-f(\ws)) \leq \frac{C_1\tau \sigma^2}{T_1+\dt_2} \\
	&\Rightarrow T_1 \leq \left(C_1\tau-1\right) \dt_2 \quad \Rightarrow \quad T_0 \leq \frac{\left(C_1\tau-1\right)^2}{C_1\tau} \dt_2.
	\end{align*}
	Repeating this argument, we can show that
	\begin{align*}
	T = T_0 \leq \frac{\left(C_1\tau-1\right)^i}{\left(C_1\tau\right)^{i-1}}  \dt_i \quad \mbox{ and } \quad T_i \leq \frac{\left(C_1\tau-1\right)^{j-i}}{\left(C_1\tau\right)^{j-i-1}}  \dt_j \quad \forall \; i < j.
	\end{align*}
	We will use $i=1$ in particular, which specializes to
	\begin{align*}
	T_1 \leq \frac{\left(C_1\tau-1\right)^{j-1}}{\left(C_1\tau\right)^{j-2}}  \dt_j \quad \forall \; j \geq 2.
	\end{align*}
	Using the above inequality, we can lower bound the sum of $\dt_j$ as
	\begin{align}
	\sum_{j=2}^{\kappa} \dt_j &\geq T_1 \cdot \sum_{j=2}^{\kappa} \frac{\left(C_1\tau\right)^{j-2}}{\left(C_1\tau-1\right)^{j-1}} \geq T_1 \cdot \frac{1}{C_1\tau} \cdot \sum_{j=2}^{\kappa} \left(1+\frac{1}{C_1\tau}\right)^{j-2} \nonumber \\ &\geq T_1\cdot \frac{1}{C_1\tau} \cdot \exp\left(\kappa/(C_1\tau)\right).\label{eqn:nsumineq}
	\end{align}
	This means that
	\begin{align*}
	\E{f(\w_{T_i})}-f(\ws) &\geq \frac{d}{2} \cdot \E{\left(\w_{T_i}^{(2)} - \left(\ws\right)^{(2)}\right)^2} \geq \exp(-6) \sigma^2 d\cdot \sum_{i=1}^{\dt_1} \eta_{T+i}^2 \\
	&\geq \frac{\exp(-6)\sigma^2 d}{\dt_1}
	\geq \frac{\exp(-6)\sigma^2 d}{T_1} \geq \frac{\exp\left(\kappa/(C_1\tau) - 3\right)}{C_1\tau} \cdot \frac{\sigma^2 d}{\sum_{j=2}^{\kappa} \dt_j},
	\end{align*}
	where we used~\eqref{eqn:nsumineq} in the last step. Rearranging, we obtain
	\begin{align*}
	\frac{\E{f(\w_{T_{\kappa}})}-f(\ws)}{\left(\sigma^2 d/T_\kappa\right)} \geq \frac{\exp\left(\kappa/(C_1\tau) - 3\right)}{C_1\tau}. 
	\end{align*}
	
	If we choose a large enough $C$ (e.g., $3C_1$), the right hand side is at least $\frac{\exp\left((C/C_1)\log (\kappa+1) - 3\right)}{\kappa} \ge \kappa$.
	
\end{proof}

\begin{proof}[Proof of theorem~\ref{thm:limsup}]
	Theorem~\ref{thm:limsup} follows as a straightforward consequence of Theorem~\ref{thm:limsup-additive} and lemma~\ref{lem:lbreduction}.
\end{proof}

\begin{theorem}\label{thm:badfraction}
	There exists universal constants $C_1,C_2 > 0$ such that
	for any $\tau \le \frac{\kappa}{CC_1\log (\kappa+1)}$ where $C$ is the constant in Theorem~\ref{thm:limsup}, for any SGD algorithm and any number of iteration $T > 0$ there exists a $T' \ge T$ such that for any $\nt \in [T',(1+1/C_2\tau)T']$ we have
	$\frac{\E{f(\w_{\nt})}-f(\ws)}{\left(\sigma^2 d/\nt\right)} \geq \tau$.
\end{theorem}

\begin{theorem}\label{thm:badfraction-additive}
	Consider the additive noise oracle setting, where, we have access to stochastic gradients satisfying:
	\begin{align*}
	\widehat{\nabla f}(\w) = \nabla f(\w) + \zeta = \Cov (\w-\ws) + \zeta,
	\end{align*}
	where, 
	\begin{align*}
	\E{\zeta|\w} = 0 \text{, and, } \E{\zeta\zeta^\top|\w} = {\sigma}^2\Cov
	\end{align*}
	There exists universal constants $C_1,C_2 > 0$ such that
	for any $\tau \le \frac{\kappa}{CC_1\log (\kappa+1)}$ where $C$ is the constant in Theorem~\ref{thm:limsup}, for any SGD algorithm and any number of iteration $T > 0$ there exists a $T' \ge T$ such that for any $\nt \in [T',(1+1/C_2\tau)T']$ we have
	$\frac{\E{f(\w_{\nt})}-f(\ws)}{\left(\sigma^2 d/\nt\right)} \geq \tau$.
\end{theorem}

To prove Theorem~\ref{thm:badfraction-additive}, we rely on the following key lemma, which says if a query point $\w_T$ is bad (in the sense that it has expected value more than $10\tau\sigma^2 d/T$), then it takes at least $\Omega(T/\tau)$ steps to bring the error back down.

\begin{lemma}\label{lem:slowdecrease}
	There exists universal constants $C_1,C_2 > 0$ such that
	for any $\tau \le \frac{\kappa}{CC_1\log (\kappa+1)}$ where $C$ is the constant in Theorem~\ref{thm:limsup}, suppose at step $T$, the query point $\w_T$ satisfies $f(\w_T)-f(\ws) \ge C_1\tau\sigma^2 d/T$, then for all $\tilde{T} \in [T,(1+\frac{1}{C_2\tau})T]$ we have $\E{f(\w_{\tilde{T}})}-f(\ws) \ge \tau\sigma^2 d/T \ge \tau\sigma^2 d/\tilde{T}$.
\end{lemma}

\begin{proof}[Proof of Lemma~\ref{lem:slowdecrease}]
	Since $f(\w_T)-f(\ws) \ge C_1\tau\sigma^2 d/T$ and\\ $f(\w_T) = \frac{d}{2}\left(\kappa \left(\w_T^{(1)} - \left(\ws\right)^{(1)}\right)^2 + \left(\w_T^{(2)} - \left(\ws\right)^{(2)}\right)^2\right)$, we know either 
	$\left(\w_T^{(1)} - \left(\ws\right)^{(1)}\right)^2 \ge C_1\tau\sigma^2/2\kappa T$ or $\left(\w_T^{(2)} - \left(\ws\right)^{(2)}\right)^2 \ge C_1\tau \sigma^2/2T$. Either way, we have a coordinate $i$ with eigenvalue $\lambda_i$ ($\kappa$ or $1$) such that $\left(\w_T^{(i)} - \left(\ws\right)^{(i)}\right)^2 \ge C_1\tau\sigma^2/(2T\lambda_i$). 
	
	Similar as before, choose $\dt$ to be the first point such that
	$$
	\eta_{T+1}+\eta_{T+2}+\cdots + \eta_{T+\dt} \in [1/\lambda_i,3/\lambda_i].
	$$
	First, by~\eqref{eqn:exp-1} or~\eqref{eqn:exp-2}, we know for any $T\le \nt \le T+\dt$, $\E{\left(\w_{\nt}^{(i)} - \left(\ws\right)^{(i)}\right)^2} \ge \exp(-6)C_1\tau\sigma^2/(2\lambda_i T)$ just by the first term. When we choose $C_1$ to be large enough the contribution to function value by this direction alone is larger than $\tau\sigma^2/T$. Therefore every query in $[T,T+\dt]$ is still bad.
	
	We will consider two cases based on the value of $S^2 := \sum_{\nt = T+1}^{T+\dt} \eta_{\nt}^2$.
	
	If $S^2 \le C_2\tau/(\lambda_i^2 T)$ (where $C_2$ is a large enough universal constant chosen later), then by Cauchy-Schwartz we know
	$$
	S^2 \cdot \dt \ge (\sum_{\nt = T+1}^{T+\dt} \eta_{\nt})^2 \ge 1/\lambda_i^2.
	$$
	Therefore $\dt \ge T/C_2\tau$, and we are done.
	
	If $S^2 > C_2\tau/(\lambda_i^2 T)$, by Equation \eqref{eqn:exp-1} and \eqref{eqn:exp-2} we know
	\begin{align*}
	\E{\left(\w_{T+\dt}^{(i)} - \left(\ws\right)^{(i)}\right)^2} &\ge \sigma^2 \sum_{\nt = T+1}^{T+\dt}\eta_{\nt}^2 \exp\left(-2\lambda_i \sum_{j=\nt+1}^{T+\dt} \eta_j \right)\\ &\ge \exp(-6)\sigma^2 \sum_{\nt = T+1}^{T+\dt}\eta_{\nt}^2 \ge \exp(-6) \cdot C_2\tau\sigma^2/(\lambda_i^2 T).
	\end{align*}
	Here the first inequality just uses the second term in Equation \eqref{eqn:exp-1} or \eqref{eqn:exp-2}, the second inequality is because $\sum_{j=\nt+1}^{T+\dt} \eta_j \le \sum_{j=T+1}^{T+\dt} \eta_j \le 3/\lambda_i$ and the last inequality is just based on the value of $S^2$. In this case as we can see as long as $C_2$ is large enough, $T+\dt$ is also a point with $\E{f(\w_{T+\dt})}-f(\ws) \ge \lambda_i \E{\left(\w_{T+\dt}^{(i)} - \left(\ws\right)^{(i)}\right)^2} \ge C_1\tau\sigma^2/(T+\dt)$, so we can repeat the argument there. Eventually we either stop because we hit case 1: $S^2 \le C_2\tau/\lambda_i^2 T$ or the case 2 $S^2 > C_2\tau/\lambda_i^2 T$ happened more than $T/C_2\tau$ times. In either case we know for any $\nt \in [T,(1+1/C_2)T]$ $\E{f(\w_{\nt})}-f(\ws) \ge \tau\sigma^2/T \ge \tau\sigma^2/\nt$ as the lemma claimed.
\end{proof}

\begin{proof}[Proof of Theorem~\ref{thm:badfraction-additive}]
Theorem~\ref{thm:badfraction-additive} is an immediate corollary of Theorem~\ref{thm:limsup-additive} and Lemma~\ref{lem:slowdecrease}.
\end{proof}

\begin{proof}[Proof of Theorem~\ref{thm:badfraction}]
	Theorem~\ref{thm:badfraction} is an immediate corollary of Theorem~\ref{thm:badfraction-additive} and lemma~\ref{lem:lbreduction}
\end{proof}

\section{Details of experimental setup}
\label{app:exp}
\subsection{Synthetic 2-d Streaming Least Squares Experiments}
\label{sapp:syn2d}
As mentioned in the main paper, we consider four condition numbers namely $\kappa\in\{50,100,200,400\}$. We run all experiments for a total of $\kappa_{\max}^2=400^2=160000$ iterations. The two eigenvalues of the Hessian are $1$ and $1/\kappa$ respectively and the noise level $\sigma^2=1$ and we average our results with five random seeds. All our grid search results for the polynomially decaying learning rates are conducted on a $8\times 8$ grid of learning rates $\times$ decay factor and whenever a best run lands at the edge of the grid, the grid is extended so that we have the best run in the interior of the grid search. For the step decay schedules, note that we fix the learning rate (details below), and vary only the decay factor.

For the $O(1/t)$ learning rate, we search for decay parameter over $8-$points log-spaced between $\{1/(200\kappa), 5000/\kappa\}$. The starting learning rate is searched over $8$ points logarithmically spaced between $\{1/\kappa,5\}$.

For the $O(1/\sqrt t)$ learning rate, the decay parameter is searched over $8$ logarithmically spaced points between $\{1/(2500\kappa),100/\kappa\}$. The starting learning rate is searched between $\{1/(10\kappa),5\}$ with $8$ logarithmically spaced points.


For the step decay schedule experiments, we kept the initial learning rate to be $0.1$ and swept over when to decay in multiples of $T/\log{T}$, i.e., vary some parameter $c\in\{0.25,0.5,0.75,1.0,1.25,1.5,2,4\}$ where the learning rate decays by a factor of $2$ every $c\cdot T/log{T}$ steps. We found that the values chosen in most experiments were very close to $1$, i.e., they were either $1$ or $1.25$ or some very rare cases, $1.5$.

With regards to the suffix iterate averaging, we used a constant stepsize of $0.1$ and averaged iterates over the final half of the iterations.

\subsection{Non-Convex experiments on \cifar~ dataset with a 44-layer residual net}
\label{sapp:noncvx}
As mentioned in the main paper, for all the experiments, we use the Nesterov's Accelerated gradient method~\citep{Nesterov83} implemented in pytorch~\footnote{\url{https://github.com/pytorch}} with a momentum set to $0.9$ and batchsize set to $128$, total number of training epochs set to $100$, $\ell_2$ regularization set to $0.0005$. 

With regards to learning rates, we consider $10-$values geometrically spaced as $\{1,0.6,\cdots,0.01\}$. To set the decay factor for any of the schemes such as~\ref{eq:p1},\ref{eq:p2}, and~\ref{eq:e1}, we use the following rule. Suppose we have a desired learning rate that we wish to use towards the end of the optimization (say, something that is $100$ times lower than the starting learning rate, which is a reasonable estimate of what is typically employed in practice), this can be used to obtain a decay factor for the corresponding decay scheme. In our case, we found it advantageous to use an additively spaced grid for the learning rate $\gamma_t$, i.e., one which is searched over a range  $\{0.0001,0.0002,\cdots,0.0009,0.001,\cdots,0.009\}$ at the $80^{th}$ epoch, and cap off the minimum possible learning rate to be used to be $0.0001$ to ensure that there is progress made by the optimization routine. For any of the experiments that yield the best performing gridsearch parameter that falls at the edge of the grid, we extend the grid to ensure that the finally chosen hyperparameter lies in the interior of the grid. All our gridsearches are run such that we separate a tenth of the training dataset as a validation set and train on the remaining $9/10^{th}$ dataset. Once the best grid search parameter is chosen, we train on the entire training dataset and evaluate on the test dataset and present the result of the final model (instead of choosing the best possible model found during the course of optimization).

\end{document}